%% file: latent_TR_nuclear_norm.tex
\documentclass[twocolumn,twoside]{IEEEtran}
\usepackage{url}
\usepackage{graphicx}
\usepackage{caption}
\usepackage{array}
\usepackage{subfig}
\usepackage{cite}
\usepackage{stmaryrd}
\usepackage{amsfonts,latexsym,amssymb}
\usepackage[cmex10]{amsmath}
\usepackage{algorithm}
\usepackage{algorithmic}
\usepackage{hyperref}
\usepackage{multirow}
\usepackage{arydshln}
\usepackage{amsthm}
\usepackage[OT1]{fontenc}
\usepackage{mathtools}
\usepackage{enumerate}
\usepackage{booktabs}
\newtheorem{definition}{{Definition}}

\newtheorem{theorem}{{Theorem}}
\newtheorem{proposition}{{Proposition}}

\begin{document}

\title{An Efficient Tensor Completion Method via New Latent Nuclear Norm}

\author{Jinshi Yu$^{1}$, Weijun Sun$^{1}$, Yuning Qiu$^{1}$, Shengli Xie$^{1}, \IEEEmembership{IEEE Fellow}$

\thanks{$^{1}$Faculty of Automation, Guangdong University of Technology, Guangzhou, 510006, China.}}

%\author{Jinshi Yu$^{1}$, Guoxu Zhou$^{1,*}$

%\thanks{$^{*}$Corresponding author. E-mail address: gx.zhou@gdut.edu.cn}
%\thanks{$^{1}$Faculty of Automation, Guangdong University of Technology, Guangzhou, 510006, China.}
%}

%\author{Jinshi Yu, Guoxu Zhou, Andrzej~Cichocki ~\IEEEmembership{IEEE Fellow},  and Shengli~Xie \IEEEmembership{IEEE Senior Member}
%
%
%\thanks{J. Yu, G. Zhou, and S. Xie are with the Faculty of Automation, Guangdong University of
%Technology, Guangzhou, 510006, China
%(E-mail: jinshi.yu@foxmail.com, gx.zhou@gdut.edu.cn, shlxie@gdut.edu.cn).}
%\thanks{A. Cichochi is from RIKEN, Japan and SKOLTECH, Moscow, Russia}}

%\captionsetup[figure]{labelsep=space}
%\captionsetup[table]{labelsep=space}

\maketitle

%double space
%\linespread{2}

\input{Abstract}
\input{Introduction}

\input{Related_works}

\input{Notations}

\input{Model}

\input{Experiments}
\input{Conclusion}

\newpage
\bibliographystyle{ieeetr}
\bibliography{NTRARbib}

\end{document}

%% file: Abstract.tex
\begin{abstract}
In tensor completion, the latent nuclear norm is commonly used to induce low-rank structure, while substantially failing to capture the global information due to the utilization of unbalanced unfolding scheme. To overcome this drawback, a new latent nuclear norm equipped with a more balanced unfolding scheme is defined for low-rank regularizer. Moreover, the new latent nuclear norm together with the Frank-Wolfe (FW) algorithm is developed as an efficient completion method by utilizing the sparsity structure of observed tensor.  Specifically, both FW linear subproblem and line search only need to access the observed entries, by which we can instead maintain the sparse tensors and a set of  small basis matrices during iteration. 
%Using the new latent nuclear norm as low-rank regularizer, the closed-form solution of FW linear subproblem can be obtained from rank-one SVD. 
Most operations are based on sparse tensors, and the closed-form solution of FW linear subproblem can be obtained from rank-one SVD.
We theoretically analyze the space-complexity and time-complexity of the proposed method, and show that it is much more efficient over other norm-based completion methods for higher-order tensors. Extensive experimental results of visual-data inpainting demonstrate that the proposed method is able to achieve state-of-the-art performance at smaller costs of time and space, which is very meaningful for the memory-limited equipment in practical applications.

%In tensor completion, the latent nuclear norm is commonly used to induce low-rank structure, while substantially failing to capture the global information due to the utilization of unbalanced unfolding scheme. To overcome this drawback, a new latent nuclear norm equipped with a more balanced unfolding scheme is defined for exploiting the low-rank structure. By utilizing the Frank-Wolfe (FW) algorithm to optimize the new latent nuclear norm, an efficient method is developed for tensor completion. Specifically, both FW linear subproblem and line search only need to access the observed entries, by which we can instead maintain the sparse tensors and a set of  small basis matrices during iteration. Most operations can be performed on sparse tensors, and the closed-form solution of FW linear subproblem obtained from rank-one SVD. These substantially reduce the space and time complexities compared with other norm-based methods, especially for higher-order tensors. Extensive experimental results of visual-data inpainting demonstrate that the proposed method is able to achieve state-of-the-art performance at smaller costs of time and space, which is very meaningful for the memory-limited equipment in practical applications.
\end{abstract}
\IEEEpeerreviewmaketitle

%% file: Introduction.tex
\section{Introduction}
In the past decades, tensor completion has aroused increasing attention due to its wide applications in a variety of fields, such as computer vision \cite{geona2019entropy, liu2019low, gandy2011tensor, kressner2014low, pang2009robust, xu2015parallel, romera2013new, kajo2018incremental}, multi-relational link prediction \cite{liu2014trace, jenatton2012latent, guo2017efficient}, and recommendation system \cite{wimalawarne2014multitask, frolov2017tensor, shang2017fuzzy, ioannidis2018coupled}. The goal of tensor completion is to recover an incomplete tensor from partially observed entries, and the most existing methods try to achieve it via the low-rank structure assumption. To our best knowledge, these tensor completion methods can mainly be categorized into tensor decomposition based method and rank-minimization based method.

Tensor decomposition based method aims to decompose the incompleted tensor into a sequence of low-rank factors and then predict the missing entries via the latent factors. For example, the CANDECOMP/PARAFAC (CP) decomposition based methods \cite{bro1998multi, acar2011scalable, sorber2013optimization, yokota2016smooth, zhao2015bayesian, zhao2016bayesian} recover the target tensor by a summation of component rank-one tensors, and the Tucker decomposition based methods \cite{chen2013simultaneous, liu2014factor, filipovic2013tucker, liu2014generalized} via a core tensor multiplied by a low-rank matrix along each mode. In recent years, the Tensor-Train and Tensor-Ring decompositions are commonly used to express the higher-order incomplete tensor by a multilinear product over a sequence of low-order latent cores \cite{yuan2017completion, yuan2018high, wang2017efficient, yuan2018higher}. Unfortunately, the tensor decomposition based method is non-convex, may suffer from the problem of local solutions. In addition, most of the tensor decomposition based methods require predefined rank, and their performance is rather sensitive to the rank selection. For the Tucker, Tensor-Train, and Tensor-Ring decompositions, the rank is defined as a vector; it, therefore, requires a computational expensive cost to find the optimal rank due to the immense selections.  

Rank-minimization based method is another type of approach to exploit the low-rank structure of incompleted tensor. Since the tensor rank minimization $\text{rank}(\cdot)$ is an NP-hard problem, a number of norms are defined as the convex surrogates of tensor rank, and the most commonly used ones are overlapped nuclear norm \cite{liu2013tensor, bengua2017efficient, Yu2019tensor} and latent nuclear norm \cite{tomioka2013convex, wang2019latent}. In \cite{liu2013tensor}, the overlapped nuclear norm via Tucker rank was first proposed by assuming all modes are low-rank, while it performs poorly when the target tensor is only low-rank in a certain mode. In contrast to the overlapped nuclear norm, the latent nuclear norm \cite{tomioka2013convex} generalizes better, especially for the tensor with only several modes low-rank. However, these two norm regularizers are based on the unbalanced mode-$k$ unfolding scheme, and therefore the unfolding matrices are usually unbalanced. For a significantly-unbalanced matrix of size $m\times n$, the matrix rank substantially fails to capture the global information of the target tensor due to the small upper bound $\min\{m, n\}$. Considering the powerful capacity of Tensor Train decomposition for representing higher-order tensors, the overlapped and latent nuclear norms via Tensor Train are proposed in \cite{bengua2017efficient} and \cite{wang2019latent}, respectively. These two norms are still based on the unbalanced unfolding scheme, i.e., $k$-mode unfolding scheme (the first $k$ modes versus the rest). Though the Tensor Ring nuclear norm \cite{Yu2019tensor} applied a more balance scheme to unfold the target tensor, a set of weighting-parameters are needed to carefully tune, which spent an expensive cost. Finally, the above-mentioned norm regularizers are commonly minimized by the alternating direction method of multipliers (ADMM) and block coordinate descent (BCD) algorithms, where the computational expensive partial-SVD operation on a large dense matrix is usually required.

To address the above-mentioned drawbacks, this paper defines a new latent nuclear norm by using a more balanced unfolding scheme, which is shown more powerful over the other norm regularizers in exploiting the low-rank global information of the target tensor. It should be noted that, though we applied the same balanced unfolding scheme as the overlapped TR nuclear norm in the new norm, it needn't additional weighting-parameters for the unfolding matrices. 
Moreover, instead of simply utilizing the expensive ADMM or BCD algorithms, the Frank-Wolfe (FW) algorithm is developed to minimize the proposed latent nuclear norm for tensor completion.
%Moreover, the new latent nuclear norm together with the Frank-Wolfe (FW) algorithm is developed to solve the tensor completion problem. 
Under the FW framework, we show that linear subproblem has a closed-form solution which can be obtained from the rank-one SVD, and most steps of the algorithm only need to access the observed entries. By utilizing sparsity of the observed tensor, we can only maintain the sparse tensor and small basis matrices instead of full-size tensors, thus require much smaller space in each iteration. Due to the proposed method operates on the sparse tensors and only need to perform rank-one SVD during iteration, it requires much smaller time-complexity over other tensor norms, which is discussed later. Furthermore, extensive experimental results of visual-data inpainting confirm that the proposed method is able to achieve state-of-the-art performance at smaller costs of time and space, which is very meaningful for the memory-limited equipment in practical applications. To sum up, the contributions of this paper are listed below:
\begin{itemize}
\item By using a more balanced unfolding scheme, a new latent nuclear norm is proposed, which is shown more powerful over other norm regularizers to exploit global information of the target tensor.
\item An efficient method, i.e. the new latent nuclear norm together with the Frank-Wolfe algorithm, is developed for tensor completion, which requires much smaller complexity over other tensor norms in terms of space and time.
\item The proposed method requires neither predefined rank nor additional weighting-parameters for the unfolding matrices and is empirically shown to achieve outstanding performance at smaller costs of time and space. This is very meaningful for the memory-limited equipment in practical applications.
\end{itemize}

The rest of this paper is organized as follows. The related works are described in Section \ref{related_works}. Notations and preliminaries required in this paper are introduced in Section \ref{Nota_and_Preli}. In Section \ref{model}, we define a new latent nuclear norm and develop an efficient Frank-Wolfe based algorithm. Moreover, the complexities of time and space are also theoretically analyzed. In Section \ref{experiments}, performance of the proposed method is investigated in synthetic data and real-world visual data. Finally, the work of this paper is concluded in section \ref{conclusion}.

%% file: Related_works.tex
\section{Related works}\label{related_works}
Our work is somewhat related to latent-norm based completion methods \cite{tomioka2013convex, wang2019latent} and Tensor-Ring based completion methods \cite{wang2017efficient, yuan2018tensor, Yu2019tensor}. In \cite{tomioka2013convex}, Tomioka et al. proposed the latent nuclear norm by mode-$k$ unfolding scheme (one mode versus the rest), and shown that it generalizes better than the overlapped nuclear norm \cite{liu2013tensor} when only several modes are low-rank. Since the mode-$k$ unfolding scheme is significantly-unbalanced, the unfolding matrix is usually unbalanced and the rank is often too small to describe the global information of target tensor.
% For a significantly-unbalanced matrix of size $m\times n$ constructed by the mode-$k$ unfolding scheme, large rank is usually required to capture the low-rank global information, while it failes due to the small upper bound $\min\{m, n\}$.
 Recently, Wang et al. \cite{wang2019latent} defined a new latent nuclear norm via Tensor Train, however it may still base on the significantly-unbalanced matrix due to the unbalanced $k$-mode unfolding scheme. In recent years, Wang et al. \cite{wang2017efficient} first applied Tensor Ring decomposition by alternating least square (TR-ALS) to incomplete data. Yuan et al. \cite{yuan2018tensor} proposed a method, named Tensor Ring low-rank factors (TRLRF), by combining nuclear norm regularization and TR decomposition. However, these two TR-decomposition based methods require a large computational complexity per iteration and thus may run out of the memory when encountering the large-scale data. Moreover, the TR-rank is defined as a vector and it is therefore very challenging to manually find the optimal rank due to the immense selections. In \cite{Yu2019tensor}, Yu et al. defined an overlapped Tensor Ring nuclear norm by a more balanced unfolding scheme and showed that it substantially improves the recovery performance in visual-data inpainting. Unfortunately, its computational complexity is still large and a set of weighting-parameters require computational expensive tuning.

In contrast, the proposed latent nuclear norm is defined via a more balanced unfolding scheme and requires neither predefined rank nor additional weighting-parameters. Moreover, the new latent nuclear norm together with the FW algorithm is developed as an efficient method, which is shown more powerful to exploit global information at smaller costs of time and space.

%% file: Notations.tex
\section{Notations and Preliminaries}\label{Nota_and_Preli}
\subsection{Notations}
This paper denotes scalars, vectors, and matrices by standard lowercase letters (e.g. $x, y, z$), boldface lowercase letters (e.g. ${\bf x, y, z}$), and bold capital letters (e.g. ${\bf X, Y, Z}$), respectively.  A tensor of order $N>3$ is denoted by Calligraphic letter, e.g. $\mathcal{X}\in\mathbb{R}^{I_{1}\times\cdots\times I_{N}}$. $\mathcal{X}(i_1, i_2, \cdots, i_N)$ or $x_{i_1, i_2, \cdots, i_N}$ represents an element of the index $(i_1, i_2, \cdots, i_N)$. $\mathcal{X}(:, i_2, \cdots, i_N)$ denotes a fiber along mode 1 and $\mathcal{X}(:, :, i_3, \cdots, i_N)$ a slice along mode 1 and mode 2, and so on. The inner product of $\mathcal{X}$ and $\mathcal{Y}$ of the same size is defined by $<\mathcal{X}, \mathcal{Y}>=\sum_{i_1,\cdots, i_N}^{I_1,\cdots, I_N}x_{i_1, i_2, \cdots, i_N}y_{i_1, i_2, \cdots, i_N}$, and the Frobenius norm of $\mathcal{X}$ can be calculated by $\|\mathcal{X}\|_{F}=\sqrt{<\mathcal{X}, \mathcal{Y}>}$.

\subsection{Preliminaries}
In this section, we briefly describe the Tensor Ring decomposition, Tensor Circular Unfolding, and their relation.

Tensor Ring decomposition \cite{zhao2016tensorring, zhao2019learning} is recently proposed to represent a higher-order tensor by a sequence of 3rd-order latent core tensors, i.e. TR-cores. Specifically, given an $N$th-order tensor $\mathcal{X}\in\mathbb{R}^{I_1\times I_2\times\cdots\times I_N}$, the TR-cores can be denoted by $\mathcal{G}_{k}\in\mathbb{R}^{R_{k-1}\times I_k\times R_{k}}$ and the TR-rank by the vector $[R_1, R_2,\cdots, R_N]^{\top}$, where $k = 1,\cdots, N$, $R_0=R_N$. Tensor Ring decomposition of $\mathcal{X}$ can be formally expressed by
\begin{align}\label{eq:TRD}
\mathcal{X}(i_1, i_2, \cdots, i_N) = Tr(\prod_{k=1}^{N}\mathcal{G}_{k}(:, i_k, :))
\end{align}
where $Tr({\cdot})$ is the matrix trace operation. More details of Tensor Ring decomposition can be seen in \cite{zhao2016tensorring, zhao2019learning}.

To efficiently exploit the global information of high-order tensors, Yu et al. \cite{Yu2018effiective, Yu2019tensor} defined a balance unfolding scheme named Tensor Circular Unfolding (TCU) in {\bf Definition \ref{def:kd_unfolding}} and described its relation with TR decomposition in {\bf Theorem \ref{theo:boundTR}}.
\begin{definition}\label{def:kd_unfolding}
(\emph{Tensor Circular Unfolding \cite{Yu2018effiective, Yu2019tensor}}) Suppose an $N$th-order tensor $\mathcal{X}\in\mathbb{R}^{I_{1}\times\dots\times I_{N}}$, the tensor circular unfolding matrix denoted by $\mathcal{X}_{<k, d>}\in\mathbb{R}^{I_{a}I_{a+1}\dots I_{k} \times I_{k+1}\dots I_{a-1}}$ can be represented by
\begin{align}
\mathcal{ X}_{<k, d>}(i_{a}i_{a+1}\dots i_{k},i_{k+1}\dots i_{a-1})=\mathcal{X}(i_{1},i_{2},\dots,i_{N})
\end{align}
where $d<N$ is a positive integer and
\begin{equation}\label{eq:t}
a = \left\{ \begin{array}{ll}
k-d+1, \ \ \quad\qquad d\le k ;\\
k-d+1+N \qquad  \text{otherwise} .
\end{array} \right.
\end{equation}
\end{definition}
\noindent The $d$ continuous modes $\{a, a+1,\cdots, k\}$ enumerate the rows of $\mathcal{X}_{<k, d>}$, and the rest modes its columns. To easily understand the Tensor Circular Unfolding scheme,  Fig. \ref{archi:kdUnfolding} illustrates the circularly-unfolding matrices $\{\mathcal{X}_{<k, 2>}\}_{k=1}^{5}$ obtained by unfolding $\mathcal{X}$ along modes $\{k-1, k\}$ specified by a red arc.
  
\begin{figure}
\begin{center}
\includegraphics[width=0.3\textwidth]{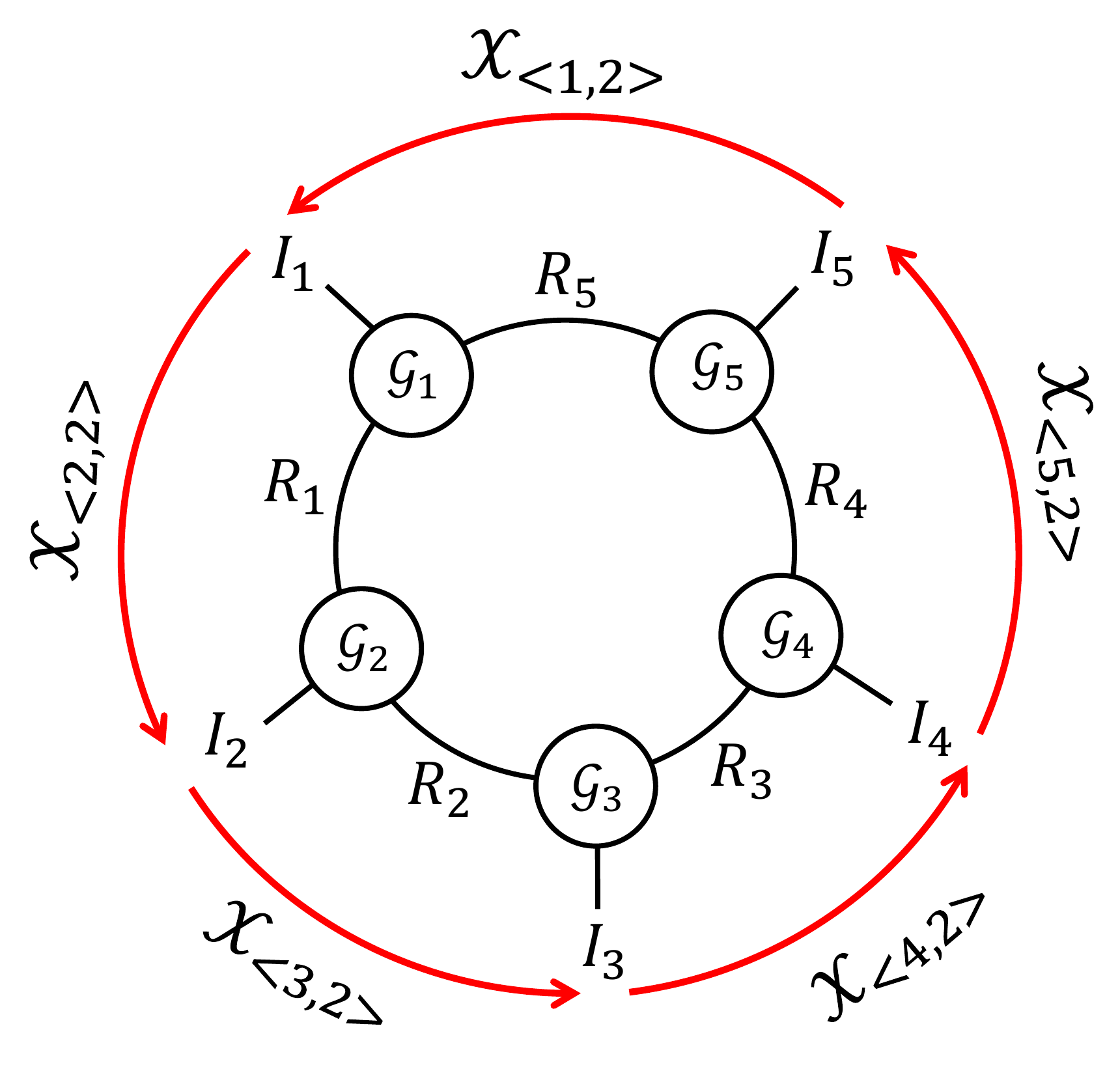}
\caption{Illustration of Tensor Ring representation of an $5$th-order tensor $\mathcal{X}\in\mathbb{R}^{I_1\times I_2\times I_3\times I_4\times I_5}$ and its Tensor Circular Unfoldings.
Each node of $\{\mathcal{G}_{k}\in\mathbb{R}^{r_{k-1}\times I_{k}\times r_{k}}\}_{k=1}^{5}$ denotes a tensor whose order decided by its number of edges. The edge connecting two nodes denotes a contraction between two tensors along a specific mode. The Tensor Circular Unfoldings $\{\mathcal{X}_{<k, 2>}\}_{k=1}^{5}$ are easily obtained by unfolding $\mathcal{X}$ along modes $\{k-1, k\}$ specified by a red arc.}\label{archi:kdUnfolding}
\end{center}
\end{figure}

\begin{theorem}\label{theo:boundTR}
Suppose $\mathcal{X}\in\mathbb{R}^{I_{1}\times\dots I_{N}}$ can be formulated by equation (\ref{eq:TRD}), then
\vspace{-0.2cm}\begin{align}
{\it rank}\left(\mathcal{X}_{<k, d>}\right)\le R_{k}R_{a-1},
\end{align}
\end{theorem}
\noindent This theorem theoretically reveals the relation of Tensor Circular Unfolding scheme and Tensor Ring decomposition, which implies that the low-rank global information can be exploited by Tensor Circular Unfolding scheme.

%% file: Model.tex
\section{Latent Tensor-Ring Nuclear Norm and Frank-Wolfe Based Alogrithm}\label{model}
\subsection{Latent Tensor-Ring Nuclear Norm}
As well-known, in tensor completion, most common definitions of the nuclear norm are overlapped nuclear norm and latent nuclear norm via Tucker/TT rank \cite{liu2013tensor, bengua2017efficient, tomioka2013convex, wang2019latent}. These nuclear norms are based on mode-$k$ unfolding scheme (one mode versus the rest) or $k$-modes unfolding scheme (the first $k$ modes versus the rest), and thus may construct significantly-unbalanced unfoldings. For a significantly-unbalanced matrix of size $m\times n$, enough large rank is usually required to describe the global information, while it fails due to the small upper bound $\min\{m, n\}$. Though TR nuclear norm \cite{Yu2019tensor} applied a more balanced unfolding scheme, i.e. Tensor Circular Unfolding (TCU), to exploit the global information and achieve a rather-well performance, its computational expensive selection of weighting-parameters seems inappropriate in practical applications. Moreover, we found that the performance of TR nuclear norm largely depends on the selection of its weighting-parameters. To solve the issues that the above mentioned nuclear norms have, a new nuclear norm named latent TR nuclear norm is defined as follow by using TCU scheme.
\begin{definition}\label{def:ltrnn}
(\emph{Latent Tensor-Ring Nuclear Norm}) Suppose an $N$th-order tensor $\mathcal{X}\in\mathbb{R}^{I_{1}\times\cdots\times I_{N}}$, the latent Tensor-Ring nuclear norm is
\begin{align}\label{eq:ltrnn}
\|\mathcal{X}\|_{ltrnn}=\min_{\sum_{k=1}^{N}\mathcal{X}_{k}=\mathcal{X}}\sum_{k=1}^{N}\|(\mathcal{X}_{k})_{<k,d>}\|_{*}
\end{align}
\end{definition}
 \noindent Note that, latent TR nuclear norm is defined as the infimum over $N$ tensors 
$\{\mathcal{X}_{k}\}_{k=1}^{N}$ which are respectively low-rank in the specific unfolding $(\mathcal{X}_{k})_{<k,d>}$.
 
Therefore, a new tensor completion model via latent Tensor-Ring nuclear norm is formulated as
\begin{align}\label{model:sltrnn}
&\min_{\mathcal{X}}\quad\|\mathcal{X}\|_{ltrnn}\nonumber\\
&s.t.\quad\mathcal{X}=\sum_{k=1}^{N}\mathcal{X}_{k}, \mathcal{X}_{\Omega} = \mathcal{T}_{\Omega}
\end{align}
where $\mathcal{T}\in\mathbb{R}^{I_{1}\times I_{2}\times\cdots\times I_{N}}$ and $\mathcal{X}^{I_{1}\times I_{2}\times\cdots\times I_{N}}$ are true tensor and reconstructed tensor, respectively. $\Omega$ denotes the index set of the observed entries, so $\mathcal{T}_{\Omega}$ represents the observed entries from the true tensor. $({\mathcal X}_{k})_{<k, d>}$ is the circularly-unfolded matrix with size $m_k\times n_k$ where $m_{k}=I_{a}I_{a+1}\cdots I_{k}$, $n_{k}=I_{k+1}I_{k+2}\cdots I_{a-1}$. Since the balanced unfolding scheme does help to catch the global information, $d$ is default set as $\left\lfloor \frac{N}{2}\right\rfloor$.

\subsection{Frank-Wolfe Based Algorithm}
Though the alternating direction method of multipliers (ADMM) and block coordinate descent (BCD) are usually used to solve the nuclear norm based completion model, they have to operate on the full-size tensors and perform partial-SVD during iterations \cite{liu2013tensor, yuan2018tensor, tomioka2010estimation, bengua2017efficient, Yu2019tensor, romera2013new, lu2019tensor}. This substantially requires large costs in time and space when encountering large-scale data. Similar to \cite{guo2017efficient}, this section instead develops the Frank-Wolfe \cite{frank1956algorithm, jaggi2013revisiting} based algorithm to solve the problem (\ref{eq:ltrnn}) by utilization of sparsity structure and rank-one SVD operation in each iteration, which will be shown much more efficient in time and space later. Under the Frank-Wolfe framework, we first transform (\ref{model:sltrnn}) into 
\begin{align}\label{FW-model}
&\min_{\mathcal{X}} F(\mathcal{X})\equiv\frac{1}{2}\| \mathcal{P}_{\Omega}(\mathcal{X})-\mathcal{P}_{\Omega}(\mathcal{T})  \|_{F}^{2}\nonumber\\
&s.t.\quad \| \mathcal{X}\|_{ltrnn}\le\beta
\end{align}
where $\mathcal{P}_{\Omega}(\mathcal{X})$ is a tensor with $\left[\mathcal{P}_{\Omega}(\mathcal{X})\right]_{i_1,\dots,i_N}=\mathcal{X}_{i_1,\dots,i_N}$ if $(i_1,\dots,i_N)\in\Omega$, and 0 otherwise. $\beta>0$ is a constraint parameter. Then we solve the problem (\ref{FW-model}) via the following three steps:
\begin{enumerate}[1)]
\item Linear subproblem $\mathcal{S}^{(t+1)}:=\text{arg}\min_{\mathcal{S}\in \mathcal{D}}<\mathcal{S}, \nabla F(\mathcal{X}^{(t)})>$.
\item Line search $\gamma^{t+1}:=\text{arg}\min_{\gamma\in[0,1]}F(\mathcal{X}^{(t)}+\gamma(\mathcal{S}^{(t+1)}-\mathcal{X}^{(t)}))$.
\item Update $\mathcal{X}^{(t+1)}:=(1-\gamma^{t+1})\mathcal{X}^{(t)}+\gamma\mathcal{S}^{(t+1)}.$
\end{enumerate}
where $\nabla F(\mathcal{X}^{(t)})$ is the gradient of $F(\mathcal{X}^{(t)})$ w.r.t. $\mathcal{X}^{(t)}$. $\mathcal{D}:=\{ \mathcal{S}\in\mathbb{T}|\| \mathcal{S} \|_{ltrnn}\le \beta \}$ is compact and convex.
\\ \\
{\bf Linear subproblem of} $\mathcal{S}^{(t+1)}$. For the linear subproblem, $\mathcal{S}^{(t+1)}:=\text{arg}\min_{\mathcal{S}\in \mathcal{D}}<\mathcal{S}, \nabla F(\mathcal{X}^{(t)})>$, {\bf Proposition \ref{Proposition:S-update}} shows that the closed-form solution can be obtained efficiently from rank-one SVD. 
\begin{proposition}\label{Proposition:S-update}
The closed-form solution of the linear problem $\mathcal{S}^{(t+1)}:=\emph{arg}\min_{\mathcal{S}\in \mathcal{D}}<\mathcal{S}, \nabla F(\mathcal{X}^{(t)})>$ can be given by
\begin{align}\label{FW-S-update}
\mathcal{S}^{(t+1)}= \emph{fold}_{k^*}(\beta{\bf u}_{k^{*}}{\bf v}_{k^{*}}^{\top})
\end{align}
where $k^{*}=\emph{arg}\max_{k\in\mathcal{D}}\sigma_{\text{max}}(-\nabla F(\mathcal{X})_{<k,d>})$, $({\bf u}_{k^{*}},{\bf v}_{k^{*}})$ denote a pair of left and right singular vectors corresponding to the largest singular value $\sigma_{\text{max}}(-\nabla F(\mathcal{X})_{<k^{*},d>})$.
\end{proposition}
\begin{proof}
Let $\|\mathcal{S}\|_{ltrnn}$ be the latent TR norm of $\mathcal{S}$, then its dual norm can be defined as 
\begin{align}
\|\nabla F(\mathcal{X})\|_{ltrnn}^{*}
&= \max_{\|S\|_{ltrnn}=1}|<\mathcal{S}, \nabla F(\mathcal{X}^{(t)})>|\nonumber\\
&= \max_{\|S\|_{ltrnn}\neq  0}\frac{|<\mathcal{S}, \nabla F(\mathcal{X}^{(t)})>|}{\|S\|_{ltrnn}}.
\end{align}
From this definition and constraint $\| \mathcal{S} \|_{ltrnn}\le \beta$, it is easy to get that
\begin{align}
<\mathcal{S}, \nabla F(\mathcal{X}^{(t)})>
&\ge-\|S\|_{ltrnn}\|\nabla F(\mathcal{X})\|_{ltrnn}^{*}\nonumber\\
&\ge-\beta\|\nabla F(\mathcal{X})\|_{ltrnn}^{*}.
\end{align}
Note that, according to \cite{tomioka2013convex}, the dual norm $\|\nabla F(\mathcal{X})\|_{ltrnn}^{*}$ can be given by
\begin{align}
\|\nabla F(\mathcal{X})\|_{ltrnn}^{*}
&=\max_{d}\|-\nabla F(\mathcal{X})_{<k,d>}\|_{\infty}\nonumber\\
&=\sigma_{{max}}(-\nabla F(\mathcal{X})_{<k^*,d>})
\end{align}
where $\|-\nabla F(\mathcal{X})_{<k,d>}\|_{2}$ denotes the spectral norm, i.e., the greatest singular value of $-\nabla F(\mathcal{X})_{<k,d>}$. Hence, 
\begin{align}
&<\mathcal{S}_{<k^*,d>}, \nabla F(\mathcal{X}^{(t)})_{<k^*,d>}>\nonumber\\
&=<\mathcal{S}, \nabla F(\mathcal{X}^{(t)})>
\ge-\beta\sigma_{{max}}(-\nabla F(\mathcal{X})_{<k^*,d>})
\end{align}
It is not difficult to find that the minimum of $<\mathcal{S}, \nabla F(\mathcal{X}^{(t)})>$ is obtained when
\begin{align}
\mathcal{S}_{<k^*,d>}=\beta{\bf u}_{k^{*}}{\bf v}_{k^{*}}^{\top}
\end{align}
Therefore, we can get $\mathcal{S}^{(t+1)}= \text{fold}_{k^*}(\beta{\bf u}_{k^{*}}{\bf v}_{k^{*}}^{\top})$.
\end{proof}

\noindent Seen from the problem (\ref{FW-model}), it is easy to check that $\nabla F(\mathcal{X}^{(t)})=\mathcal{P}_{\Omega}(\mathcal{X})-\mathcal{P}_{\Omega}(\mathcal{T})$, and its rank-one SVD can be computed efficiently by the power method in \cite{halko2011finding}.
\\  \\
{\bf Line search of} $\mathcal{\gamma}^{(t+1)}$. With $F$ in problem (\ref{FW-model}), the step-size $\gamma^{t+1}$ can be given by solving the following problem:
\begin{align}\label{FW-gamma}
\gamma^{t+1}:=\text{arg}\min_{\gamma\in[0,1]}\| \mathcal{P}_{\Omega}( \mathcal{X}^{(t)}+\gamma(\mathcal{S}^{(t+1)}-\mathcal{X}^{(t)}) - \mathcal{T}) \|_{F}^{2}
\end{align}
Note that the problem (\ref{FW-gamma}) is essentially a quadratic equation of $\gamma$, i.e.,
\begin{align}\label{FW-gamma-quadratic}
\gamma^{t+1}:=\text{arg}\min_{\gamma\in[0,1]}(\hat{a}\gamma^{2}+\hat{b}\gamma+\hat{c})
\end{align}
where $\hat{a}=\| \mathcal{P}_{\Omega}(\mathcal{S}^{(t+1)}-\mathcal{X}^{(t)}) \|_{F}^{2}$, $\hat{b}=2< \mathcal{P}_{\Omega}(\mathcal{X}^{(t)}-\mathcal{T}) ,\mathcal{P}_{\Omega}(\mathcal{S}^{(t+1)}-\mathcal{X}^{(t)}) >$, $\hat{c}=\| \mathcal{P}_{\Omega}(\mathcal{X}^{(t)}-\mathcal{T})\|_{F}^{2} $.
Hence, it is easy to get a simple closed-form solution:
\begin{align}\label{FW-gamma-update}
\gamma^{(t+1)} =
\left\{ \begin{array}{lll}
0                        &-\frac{\hat{b}}{2\hat{a}}\in(-\infty,0);\\[0.05cm]
-\frac{\hat{b}}{2\hat{a}}   & -\frac{\hat{b}}{2\hat{a}}\in[0,1];\\[0.05cm]
1                        &-\frac{\hat{b}}{2\hat{a}}\in(1,+\infty)
\end{array} \right.
\end{align}
\\
{\bf Update} $\mathcal{X}^{(t+1)}$. Note that the update of $\gamma^{(t+1)}$ only needs to access the entries indexed by $\Omega$, i.e., $\mathcal{S}_{\Omega}^{(t+1)},\mathcal{X}_{\Omega}^{(t)}$. Hence, instead of calculating and storing the full tensors during iterations, we can follow an efficiently update scheme proposed in \cite{guo2017efficient}. This efficiently update scheme consists of two steps. The step 1 is to only store the sparse tensors $\mathcal{S}^{(t+1)}_{\Omega}, \mathcal{X}^{(t+1)}_{\Omega}$ and the basis matrices $\{{\bf U}_{k}\in\mathbb{R}^{m_{k}\times R_{k}}, {\bf \Sigma}_{k}\in\mathbb{R}^{R_{k}\times R_{k}}, {\bf V}_{k}\in\mathbb{R}^{n_{k}\times R_{k}}\}_{k=1}^{N}$ satisfied that $\mathcal{X}^{(t+1)}=\sum_{k=1}^{N}\text{fold}_{k}({\bf U}_{k}{\bf \Sigma}_{k}{\bf V}_{k}^{\top})$. Specifically,
\begin{align}
&\quad\mathcal{S}^{(t+1)}_{\Omega}= \left (\text{fold}_{k}(\beta{\bf u}_{k^{*}}{\bf v}_{k^{*}}^{\top})\right)_{\Omega}\label{FW-SOmega-update}\\[0.1cm]
&\quad\mathcal{X}^{(t+1)}_{\Omega}=(1-\gamma^{(t+1)})\mathcal{X}^{(t)}_{\Omega}+\gamma^{(t+1)}\mathcal{S}^{(t+1)}_{\Omega}\label{FW-XOmega-update}\\[0.1cm]
&\left\{\begin{array}{llll}
{\bf\Sigma}_{k}=(1-\gamma^{(t+1)}){\bf\Sigma}_{k},\quad k\ne k^{*},\\[0.05cm]
{\bf\Sigma}_{k^{*}}=\left[ \begin{array}{cc}
(1-\gamma^{(t+1)}){\bf\Sigma}_{k^{*}} & {\bf 0}\\[0.05cm]
{\bf 0} & \gamma^{t+1}\beta
\end{array} 
\right ],\\
{\bf U}_{k^{*}}=[ \begin{array}{cc}{\bf U}_{k^{*}} &{\bf u}_{k^{*}} \end{array} ],\\[0.05cm]
{\bf V}_{k^{*}}=[ \begin{array}{cc}{\bf V}_{k^{*}} &{\bf v}_{k^{*}} \end{array} ],
\end{array}\right.\label{FW-basis-update}
\end{align}
where $\{ {\bf U}_{k}, {\bf \Sigma}_{k},{\bf V}_{k}\}_{k=1}^{N}$ are initialized to empty matrices. It is not difficult to check that the above formulas satisfy the update of $\mathcal{X}^{(t+1)}:=(1-\gamma^{(t+1)})\mathcal{X}^{(t)}+\gamma^{(t+1)}\mathcal{S}^{(t+1)}$.
Step 2 is using a trick shown in {\bf Algorithm \ref{alg:ReducingSize}} to reduce the size of the basis matrices without considerably increasing the objective function value $F(\mathcal{X})$ when $\sum_{N}^{k=1}R_{k}>\bar{R}$, where $\bar{R}$ is a given threshold. This trick avoids the problem that the basis matrices gradually increase in size and then cause memory-explosion.

We summarize the complete procedure in {\bf Algorithm \ref{alg:FW}}. Since the algorithm only accesses the observed entries of $\mathcal{S},\mathcal{X}$ and require rank-one SVD operation, it is efficient in terms of both space and time. %Moreover, as shown in \cite{guo2017efficient}, the algorithm can converge to the optimal solution at a rate of $\mathcal{O}(1/t)$, where $t$ is the number of iterations.

\renewcommand{\algorithmicrequire}{\textbf{Input:}}
\renewcommand{\algorithmicensure}{\textbf{Output:}}
\begin{algorithm}[htb]
\caption{Reducing the size of basis matrices.}
\label{alg:ReducingSize}
\begin{algorithmic}[1]
\REQUIRE
       { $\{ {\bf U}_{k}, {\bf \Sigma}_{k},{\bf V}_{k}\}_{k=1}^{N}$; }

 \STATE{ Initialize: zero filled $\mathcal{X}$ with $\mathcal{X}_{\Omega} =\mathcal{T}_{\Omega}$}, random initialized $\mathcal{X}_{k}$, $\mathcal{M}_{k}=0$, $\mathcal{Y}_{k}=0$, $\mathcal{Y}=0$.

 \FOR{$k=1 $ \TO $N$}
 \STATE{$[{\bf Q_{U}}, {\bf R_{U}}]=\text{QR}({\bf U}_{k})$, $[{\bf Q_{V}}, {\bf R_{V}}]=\text{QR}({\bf V}_{k})$;}
 \STATE{${\bf J}_{0}={\bf R_{U}}{\bf \Sigma}_{k}{\bf R}_{\bf V}^{\top}$, $\mathcal{B}_{k}=\mathcal{P}_{\Omega}(\sum_{l\ne k}\text{fold}_{l}({\bf U}_{l}{\bf \Sigma}_{l}{\bf V}_{l}))$;}
 \STATE{${\bf J}=\text{arg}\min_{{\bf J}:\|{\bf J}\|_{*}\le \|{\bf J}_{0}\|_{*}}\| \mathcal{P}_{\Omega}(\text{fold}_{k}({\bf Q_{U}}{\bf J}{\bf Q}_{V}^{\top}) +\mathcal{B}_{k})\|_{F}^{2}$;}
 \STATE{${\bf U_{J}}{\bf \Sigma_{J}}{\bf V}_{\bf J}^{\top}=\text{SVD}({\bf J}) $;}
 \STATE{$ {\bf U}_{k}={\bf Q_{U}}{\bf U_{J}}, {\bf V}_{k}={\bf Q_{V}}{\bf V_{J}}, {\bf \Sigma}_{k}={\bf \Sigma_{J}}$;}
 \STATE{$R_{k}=$ number of nonzero elements in ${\bf \Sigma_{J}}$}
 \ENDFOR
\ENSURE
       {$\{ {\bf U}_{k}, {\bf \Sigma}_{k},{\bf V}_{k}\}_{k=1}^{N}, R_k$.}
\end{algorithmic}
\end{algorithm}

\renewcommand{\algorithmicensure}{\textbf{Parameters:}}
\begin{algorithm}[htb]
\caption{FW-based algorithm for latent Tensor-Ring nuclear norm minimization.}
\label{alg:FW}
\begin{algorithmic}[1]
\REQUIRE
       { Partically observed entries ${\mathcal T}_{\Omega}$, 
\ENSURE
       $\bar{R}, tol = 10^{-5}$ .}

 \STATE{ Initialize: $\mathcal{X}^{(0)}=0$, $R_1=R_2=\cdots=R_{k}=0$, $\{ {\bf U}_{k}, {\bf \Sigma}_{k},{\bf V}_{k}\}_{k=1}^{N}=[]$.}
 \FOR{$t=1 $ \TO$t_{max}$}
 \STATE{$k^{*}=\text{arg}\max_{k\in\mathcal{D}}\sigma_{\text{max}}(-\nabla F(\mathcal{X})_{<k,d>})$}
 \STATE{$({\bf u}_{k^{*}},{\bf v}_{k^{*}})=$ a pair of left and right singular vectors corresponding to the largest singular value of $-\nabla F(\mathcal{X})_{<k,d>}$;}
 \STATE{Update $\mathcal{S}_{\Omega}^{(t+1)}$ by Equation (\ref{FW-SOmega-update});}
 \STATE{Update $\gamma^{(t+1)}$ by Equation (\ref{FW-gamma-update}); }
 \STATE{Update $\mathcal{X}_{\Omega}^{(t+1)}$ by Equation (\ref{FW-XOmega-update});}
 \STATE{Update $\{ {\bf U}_{k}, {\bf \Sigma}_{k},{\bf V}_{k}\}_{k=1}^{N}$ by Equation (\ref{FW-basis-update})}
 \STATE{$R_{k^{*}}=R_{k^{*}}+1$;}
\IF{$\sum_{N}^{k=1}R_{k}>\bar{R}$}
\STATE{Reducing the size of $\{ {\bf U}_{k}, {\bf \Sigma}_{k},{\bf V}_{k}\}_{k=1}^{N}$ by \bf{Algorithm \ref{alg:ReducingSize}};} 
\STATE{$\mathcal{X}_{\Omega}^{(t+1)}=\left(\sum_{k=1}^{N}\text{fold}_{k}({\bf U}_{k}{\bf \Sigma}_{k}{\bf V}_{k}^{\top})\right)_{\Omega}$;}
\ENDIF
 \IF{$\|\mathcal{X}_{\Omega}^{(t+1)}-\mathcal{X}_{\Omega}^{(t)}\|_{F}/\| \mathcal{X}_{\Omega}^{(t)}\|_{F}\le tol$}
 \STATE{break}
 \ENDIF
 \ENDFOR
 \STATE{Return $\mathcal{X}=\sum_{k=1}^{N}\text{fold}_{k}({\bf U}_{k}{\bf \Sigma}_{k}{\bf V}_{k}^{\top})$}
\end{algorithmic}
\end{algorithm}
\subsection{Analysis of Space-Complexity and Time-Complexity}
It is well-known that the complexities of space and time are very important to evaluate one algorithm. In this section, for an $N$th-order tensor $\mathcal{X}$ with size $I\times I\times \cdots\times I$, we aim to analyze the proposed method in terms of space complexity and time complexity. Seen from the {\bf Algorithm \ref{alg:FW}}, all the operations are based on the sparse tensors of $\|\Omega\|_{1}$ observed entries  and a set of basis matrices $\{{\bf U}_{k}\in\mathbb{R}^{I^{d}\times R_{k}}, {\bf \Sigma}_{k}\in\mathbb{R}^{R_{k}\times R_{k}}, {\bf V}_{k}\in\mathbb{R}^{I^{N-d}\times R_{k}}\}_{k=1}^{N}$. Thus, the space complexity of the proposed method is $\mathcal{O}((I^{d}+I^{N-d}+1)R+\|\Omega\|_{1})$ per iteration, where $R=\sum_{k=1}^{N}R_{k}$, $d = \left\lfloor \frac{N}{2}\right\rfloor$. For the time-complexity of the proposed method, the main per-iteration cost lies in the update of $\mathcal{S}_{\Omega}^{(t+1)}$ which consist of the rank-one SVDs of  $-\nabla F(\mathcal{X})_{<k,d>}\in\mathbb{R}^{I^{d}\times I^{N-d}}$ for $k=1,\cdots, N$ and the computation of Equation (\ref{FW-SOmega-update}). The rank-one SVDs performed by the power method require a cost of $\mathcal{O}(N(I^{N}+I^{d}+I^{N-d}))$, and the time-cost of Equation (\ref{FW-SOmega-update}) is $\mathcal{O}(\|\Omega\|_{1})$. Therefore, the overall time-complexity of the proposed method is $\mathcal{O}(N(I^{N}+I^{d}+I^{N-d}))$ per iteration.

TABLE \ref{tab:complexity} summarizes the space-complexity and time-complexity of  other tensor-norm based algorithms: (i) Overlapped nuclear norm HaLTRC \cite{liu2013tensor}; (ii) Overlapped nuclear norm via tensor-train SiLRTC-TT \cite{bengua2017efficient}; (iii)  Tensor nuclear norm TNN \cite{zhang2014novel}; (iv) Overlapped Tensor-ring nuclear norm TRNNM \cite{Yu2019tensor}; (v)  Scaled latent nuclear norm FFWTensor \cite{guo2017efficient}. Since HaLTRC and TRNNM impose $N$ auxiliary variables and $N$ Lagrangian multipliers to simplify the optimization, they both require a space-complexity of $\mathcal{O}((2N+1)I^{N}) $ per iteration. TNN requires two additional variables, and SiLRTC-TT has $N$ auxiliary variables. Thus their per-iteration space-complexities are $\mathcal{O}(3I^{N}) $ and $\mathcal{O}((N+1)I^{N}) $, respectively. Similar to the proposed algorithm, FFWTensor only needs to store the sparse tensors and a set of basis matrices at a cost of $\mathcal{O}((I^{N-1}+I+1)R+\|\Omega\|_{1})$ per iteration. And the per-iteration time-complexity of these algorithms can be obtained according to the corresponding papers.

Seen from TABLE \ref{tab:complexity}, it is not difficult to observe that LTRNNM requires a much smaller space-complexity over the other compared algorithms when the target tensor $\mathcal{X}$ has a high missing ratio and $R<<I^d$. This is because the sparsity structure of $\mathcal{X}$ is efficiently used in LTRNNM. When $N=3$, LTRNNFW reduces to the unscaled version of  FFWTensor, thus they have the same space-complexity. It is worthy noting that LTRNNM requires much lesser storage space over FFWTensor when $N > 3$, due to $(I^{d}+I^{N-d})$ is significantly smaller than $(I^{N-1}+I)$.

Note that, both LTRNNFW and FFWTensor have the smaller order of magnitude of time-complexity than the other compared algorithms, which is benefit from the sparsity structure of the target tensor and the efficient rank-one SVD used during iterations. In contrast, other algorithms have to operate on the full-sized tensors and perform partial-SVD in each iteration. Typically, performing rank-one SVD is much significantly faster than partial-SVD, especially for the large scale matrix. Therefore, it is not surprising that LTRNNFW and FFWTensor are time-efficient.

\begin{table}[htbp]
  \centering
  \caption{Space-Complexity and Time-Complexity of algorithms for one iteration.}
 \renewcommand{\arraystretch}{1.5}
  \scalebox{0.85}{
    \begin{tabular}{lll}
    \multicolumn{1}{l}{Algorithms} & \multicolumn{1}{l}{Space-Complexity} & \multicolumn{1}{l}{Time-Complexity} \\
\midrule
    HaLRTC &     $\mathcal{O}((2N+1)I^{N}) $     & $\mathcal{O}(NI^{N+1})$   \\
    SiLRTC-TT &       $\mathcal{O}((N+1)I^{N}) $     & $\mathcal{O}(I^{N+d}+I^{N+d-1})$ \\
    TNN   &          $\mathcal{O}(3I^{N}) $     & $\mathcal{O}(I^{N}log(I^{N-2})+I^{N+1})$  \\
    TRNNM &      $\mathcal{O}((2N+1)I^{N}) $    & $\mathcal{O}(NI^{N+d})$  \\
    FFWTensor &  $\mathcal{O}((I^{N-1}+I+1)R+\|\Omega\|_{1})$     &$\mathcal{O}(N(I^{N}+I^{N-1}+I))$  \\
    {\bf LTRNNFW} &   $\mathcal{O}((I^{d}+I^{N-d}+1)R+\|\Omega\|_{1})$    & $\mathcal{O}(N(I^{N}+I^{d}+I^{N-d}))$ \\
\midrule
    \end{tabular}}%
  \label{tab:complexity}%
\end{table}%

%% file: Experiments.tex
\section{experiments}\label{experiments}
\subsection{Effect of $\beta$ for the Proposed Method}
\begin{figure}[htbp]
\centering
\includegraphics[ width=0.45\textwidth]
{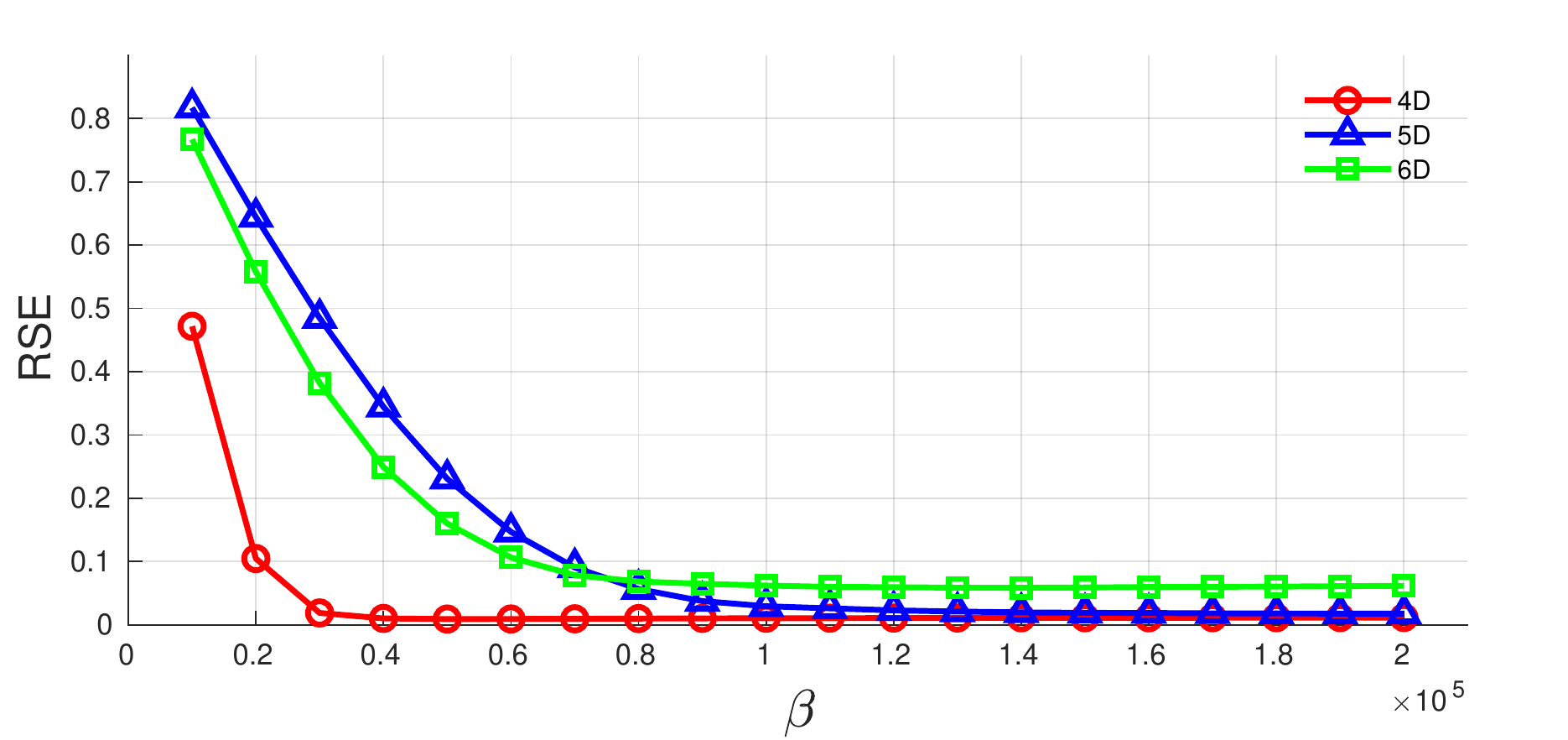}
\caption{Plots of relative square error versus $\beta$.}
\label{plot:beta}
\end{figure}
This section aims to investigate the effect of the constraint parameter $\beta$ for the proposed method on the synthetic data $\mathcal{X}\in\mathbb{R}^{I_{1}\times I_{2}\times\cdots\times I_{N}}$ with the latent structure of $\mathcal{X}=\sum_{k=1}^{N}\mathcal{X}_{k}$. All the $\{\mathcal{X}_{k}\}_{k=1}^{N}$ are generated such that $(\mathcal{X}_{k})_{<k,d>}\in\mathbb{R}^{m_k\times n_k}$ has a low-rank structure, i.e. $(\mathcal{X}_{k})_{<k,d>}={\bf A}{\bf B}^{\top}$, where the values of ${\bf A}\in\mathbb{R}^{m_{k}\times r_{k}}$ and ${\bf B}\in\mathbb{R}^{n_{k}\times r_{k}}$ are drawn randomly from the standard Gaussian distribution $\mathcal{N}(0,1)$. For simplicity, we set the dimension of each mode same and so does the corresponding low-ranks, i.e., $I_{1}= I_{2}=\cdots= I_{N}=I$, $R_1=R_2=\cdots=R_N$. The uniformly random missing ratio of 50\% is considered in this experiment, and the relative squared error (RSE) is used as the evaluation index. The RSE between the estimation $\bar{\mathcal{X}}$ and the true one $\mathcal{X}$ is defined by $\text{RSE}=\| \mathcal{X}-\bar{\mathcal{X}}\|_{F}/\|\mathcal{X}\|_{F}$.

Fig. \ref{plot:beta} shows the plots of RSE versus $\beta$ for tensors of different size $30\times30\times30\times30$ (4D), $20\times20\times20\times20\times20$ (5D), $10\times10\times10\times10\times10\times10$ (6D) and corresponding rank tuples $(5, 5, 5, 5)$ (4D),  $(6, 6, 6, 6, 6)$ (5D), $(7, 7, 7, 7, 7, 7)$ (6D). The plots illustrate that the proposed method is robust to constraint parameter $\beta$ in a wide range, which is an important property for algorithms in practical applications.

\subsection{Performance in High-Order Form}
\begin{table*}[htbp]
  \centering
  \caption{Performance (RSE, PSNR, SSIM, SSDI and RunTime) of FFWTensor and LTRNNFW under different-order form \{3D, 6D, 9D, 12D\} and missing ratios \{ 70\%, 75\%, 80\%, 85\%, 90\%, 95\%\}.}
  \renewcommand{\arraystretch}{1.2}
  \scalebox{1}{
    \begin{tabular}{c|cccccc|ccccc}
     \hline
          & \multicolumn{6}{c|}{FFWTensor}                & \multicolumn{5}{c}{\bf LTRNNFW} \\
      \hline
    $mr$    &       & RSE   & PSNR  & SSIM  & SSDI (1e6) & RunTime (s) & RSE   & PSNR  & SSIM  & SSDI (1e6) & RunTime (s) \\
      \hline
    \multirow{4}[2]{*}{70\%} & 3D    & 0.0281 & 44.84 & 0.9894 & 7.79  & 53.71 & 0.028 & 44.88 & 0.9896 & 7.77  & 58.3 \\
          & 6D    & 0.2442 & 26.08 & 0.7828 & 37.79 & 168.05 & \textbf{0.0144} & \textbf{50.64} & \textbf{0.9953} & \textbf{3.43} & \textbf{120.07} \\
          & 9D    & 0.4914 & 20.01 & 0.6455 & 67.06 & 304.91 & \textbf{0.0076} & \textbf{56.18} & \textbf{0.9984} & \textbf{4.72} & \textbf{260.39} \\
          & 12D   & 0.4249 & 21.27 & 0.7436 & 94.71 & 525.01 & \textbf{0.0207} & \textbf{47.52} & \textbf{0.9883} & \textbf{3.75} & \textbf{420.37} \\
     \hline
    \multirow{4}[2]{*}{75\%} & 3D    & 0.0396 & 41.88 & 0.9811 & 6.78  & 48.3  & 0.0392 & 41.98 & 0.9821 & 4.99  & 49.96 \\
          & 6D    & 0.3081 & 24.06 & 0.7565 & 29.32 & 142.25 & \textbf{0.0187} & \textbf{48.38} & \textbf{0.9934} & \textbf{1.3} & \textbf{89.96} \\
          & 9D    & 0.5646 & 18.8  & 0.6264 & 63.99 & 275.87 & \textbf{0.0112} & \textbf{52.87} & \textbf{0.9973} & \textbf{2.49} & \textbf{143.32} \\
          & 12D   & 0.4991 & 19.87 & 0.7263 & 97.86 & 578.55 & \textbf{0.0256} & \textbf{45.66} & \textbf{0.9838} & \textbf{1.54} & \textbf{396.45} \\
    \hline
    \multirow{4}[2]{*}{80\%} & 3D    & 0.0564 & 38.81 & 0.9675 & 7.26  & 35.01 & 0.0563 & 38.82 & 0.9682 & 7.75  & 36.23 \\
          & 6D    & 0.3788 & 22.26 & 0.7145 & 33.78 & 128.21 & \textbf{0.0223} & \textbf{46.85} & \textbf{0.9912} & \textbf{2.74} & \textbf{73.63} \\
          & 9D    & 0.6407 & 17.7  & 0.6069 & 55.36 & 265.15 & \textbf{0.016} & \textbf{49.72} & \textbf{0.9954} & \textbf{3.85} & \textbf{99.19} \\
          & 12D   & 0.6028 & 18.23 & 0.6998 & 108.82 & 473.6 & \textbf{0.0327} & \textbf{43.54} & \textbf{0.9783} & \textbf{2.85} & \textbf{196.13} \\
      \hline
    \multirow{4}[2]{*}{85\%} & 3D    & 0.0858 & 35.17 & 0.9387 & 7.74  & 27.33 & 0.084 & 35.35 & 0.942 & 6.37  & 27.42 \\
          & 6D    & 0.4974 & 19.9  & 0.6745 & 29.63 & 103.99 & \textbf{0.0309} & \textbf{44.03} & \textbf{0.9855} & \textbf{1.34} & \textbf{69.89} \\
          & 9D    & 0.7192 & 16.7  & 0.5957 & 51.71 & 225.51 & \textbf{0.0249} & \textbf{45.91} & \textbf{0.991} & \textbf{2.49} & \textbf{68.78} \\
          & 12D   & 0.7203 & 16.68 & 0.6665 & 112.21 & 435.47 & \textbf{0.0453} & \textbf{40.72} & \textbf{0.9663} & \textbf{1.45} & \textbf{189.95} \\
      \hline
    \multirow{4}[2]{*}{90\%} & 3D    & 0.1398 & 30.92 & 0.8853 & 5.53  & 20.14 & 0.1403 & 30.89 & 0.8862 & 5.51  & 20.61 \\
          & 6D    & 0.6372 & 17.75 & 0.6359 & 34.75 & 95.95 & \textbf{0.0494} & \textbf{39.95} & \textbf{0.9703} & \textbf{2.04} & \textbf{65.39} \\
          & 9D    & 0.7988 & 15.78 & 0.5994 & 49.68 & 212.35 & \textbf{0.0432} & \textbf{41.12} & \textbf{0.9792} & \textbf{3.06} & \textbf{41.92} \\
          & 12D   & 0.852 & 15.22 & 0.6259 & 120.37 & 357.26 & \textbf{0.0712} & \textbf{36.79} & \textbf{0.9409} & \textbf{2.13} & \textbf{101.78} \\
      \hline
    \multirow{4}[1]{*}{95\%} & 3D    & 0.2727 & 25.12 & 0.7671 & 6.74  & 12.99 & 0.2748 & 25.05 & 0.7649 & 5.39  & 13.17 \\
          & 6D    & 0.8214 & 15.54 & 0.6095 & 29.94 & 73.61 & \textbf{0.1017} & \textbf{33.69} & \textbf{0.9254} & \textbf{1.34} & \textbf{17.61} \\
          & 9D    & 0.8994 & 14.75 & 0.6235 & 49.77 & 158.51 & \textbf{0.0987} & \textbf{33.95} & \textbf{0.9328} & \textbf{2.23} & \textbf{20.03} \\
          & 12D   & 0.9566 & 14.22 & 0.6073 & 125.97 & 309.57 & \textbf{0.1387} & \textbf{30.99} & \textbf{0.8732} & \textbf{1.44} & \textbf{36.3} \\
          \hline
     \end{tabular}}
  \label{table:MRI}
\end{table*}

To the best of our knowledge, reshaping low-order tensors into high-order tensors is a common practice to improve the performance for TT/TR-based methods on visual-data completion \cite{wang2017efficient, bengua2017efficient, yuan2018high, Yu2018effiective, Yu2019tensor}. To evaluate the proposed method in high-order form, the first 180 frames of the brain Magnetic Resonance Imaging (MRI) \cite{liu2013tensor} with cropped size $180\times 216$ is considered in this experiment. Thus, we present the MRI data by the 3rd-order tensor of size $180\times216\times180$ and further reshape into tensors of size $12\times15\times12\times18\times12\times15$ (6D), $4\times5\times9\times4\times6\times9\times4\times5\times9$ (9D) and  $4\times5\times3\times3\times4\times6\times3\times3\times4\times5\times3\times3$ (12D). RSE, peak signal-to-noise ratio (PSNR), structural similarity (SSIM) \cite{wang2004image}, storage size during iteration (SSDI) and RunTime are used to evaluate the performance. The PSNR between the estimation $\bar{\mathcal{X}}$ and the true one $\mathcal{X}$ is defined by $\text{PSNR}= 10\log_{10}(255^2/\text{MSE})$, where $\text{MSE} = \|\mathcal{X}- \bar{\mathcal{X}}\|_{F}^{2}/\text{num}(\mathcal{X})$ and $\text{num}(\mathcal{X})$ denotes the number of entries of $\mathcal{X}$. We choose FFWTensor method to be the baseline, due to it and the proposed method both took full advantage of the sparsity structure of the observed tensor during iterations. For simplicity, the SSDI of both FFWTensor and the proposed method is defined by a sum of the total number of entries of basis matrices $\{{\bf U}_{k}\in\mathbb{R}^{p_k\times r_k}, {\bf \Sigma}_{k}\in\mathbb{R}^{r_k\times r_k}, {\bf V}_{k}\in\mathbb{R}^{r_k\times q_k}\}_{k=1}^{N}$ and the number of observed entries, i.e., $\text{SSDI}= \sum_{k=1}^{N}(p_k r_k+r_k+q_k r_k)+\|\Omega\|_{1}$. 

TABLE \ref{table:MRI} shows the performance of FFWTensor and LTRNNFW under different-order form \{3D, 6D, 9D, 12D\} and missing ratios \{70\%, 75\%, 80\%, 85\%, 90\%, 95\%\}. Obviously, the proposed method obtains significantly better results in the high-order form \{6D, 9D\}, while slightly degrades the performance after further reshaping into 12D form. This implies that reshaping low-order tensor to high-order tensor does help to improve the performance, especially when reshaping into an appropriate high-order form. However, FFWTensor achieves the worse performance after reshaping into high-order form. In addition, it can be observe that:
\begin{itemize} 
\item In 3D case, the proposed method obtains similar results as FFWTensor, which is caused by that the proposed method reduces to the unscaled version of FFWTensor when encountering 3rd-order tensors.
\item In high-order cases, i.e. \{6D, 9D, 12D\}, the proposed method significantly obtains better results over FFWTensor in terms of RSE, PSNR, SSIM, SSDI and RunTime. Note that, the main difference of the proposed method from FFWTensor is a more balanced unfolding scheme applied in the proposed method. Better results of \{RSE, PSNR, SSIM\} illustrate the powerful ability of the balanced unfolding scheme in catching the global information. Smaller values of SSDI and RunTime imply stronger power of data- representation and more space-and-time efficiency, which is meaningful when encountering large-scale data or the memory is limited.
\end{itemize}
Moreover, as shown in Fig. \ref{fig:MRI}, the recovery frame by the proposed method is more clear than that by FFWTensor. All these results show the superiority of the proposed method in processing the high-order tensors.

\begin{figure*}
\begin{center}
\begin{minipage}[b]{0.17\textwidth}\centering
\includegraphics[width=1\linewidth]{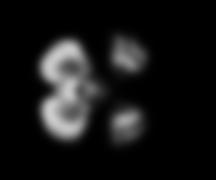}
\end{minipage}\vspace{0.1cm}
\begin{minipage}[b]{0.17\textwidth}\centering
\includegraphics[width=1\linewidth]{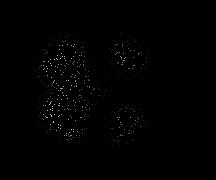}
\end{minipage}
\begin{minipage}[b]{0.17\textwidth}\centering
\includegraphics[width=1\linewidth]{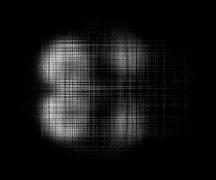}
\end{minipage}
\begin{minipage}[b]{0.17\textwidth}\centering
\includegraphics[width=1\linewidth]{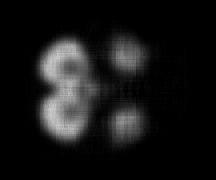}
\end{minipage}\\
\begin{minipage}[b]{0.17\textwidth}\centering
\includegraphics[width=1\linewidth]{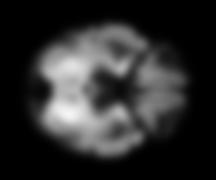}
\end{minipage}\vspace{0.1cm}
\begin{minipage}[b]{0.17\textwidth}\centering
\includegraphics[width=1\linewidth]{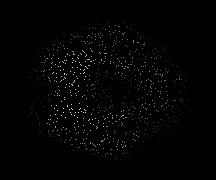}
\end{minipage}
\begin{minipage}[b]{0.17\textwidth}\centering
\includegraphics[width=1\linewidth]{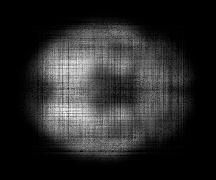}
\end{minipage}
\begin{minipage}[b]{0.17\textwidth}\centering
\includegraphics[width=1\linewidth]{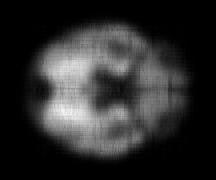}
\end{minipage}\\
\begin{minipage}[b]{0.17\textwidth}\centering
\includegraphics[width=1\linewidth]{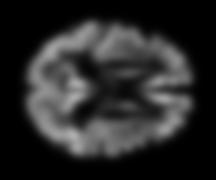}
\end{minipage}\vspace{0.1cm}
\begin{minipage}[b]{0.17\textwidth}\centering
\includegraphics[width=1\linewidth]{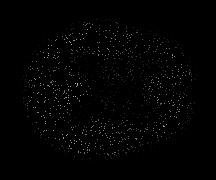}
\end{minipage}
\begin{minipage}[b]{0.17\textwidth}\centering
\includegraphics[width=1\linewidth]{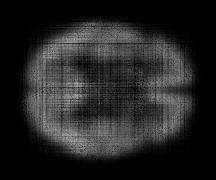}
\end{minipage}
\begin{minipage}[b]{0.17\textwidth}\centering
\includegraphics[width=1\linewidth]{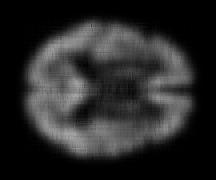}
\end{minipage}\\
\begin{minipage}[b]{0.17\textwidth}\centering
\includegraphics[width=1\linewidth]{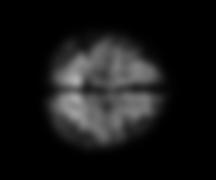}\caption*{Original}
\end{minipage}\vspace{0.1cm}
\begin{minipage}[b]{0.17\textwidth}\centering
\includegraphics[width=1\linewidth]{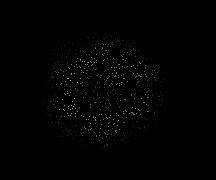}\caption*{Observation}
\end{minipage}
\begin{minipage}[b]{0.17\textwidth}\centering
\includegraphics[width=1\linewidth]{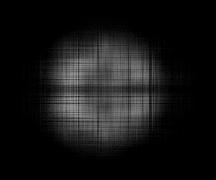}\caption*{FFWTensor}
\end{minipage}
\begin{minipage}[b]{0.17\textwidth}\centering
\includegraphics[width=1\linewidth]{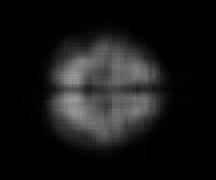}\caption*{\bf LTRNNFW}
\end{minipage}
\caption{The visual results of FFWTensor and LTRNNFW on the MRI images with the uniformly missing ratio of 95\%. The recovery results are shown by randomly picking slices.}
\label{fig:MRI}
\end{center}
\end{figure*}

\subsection{Visual Data Inpainting}
In this section, we compare the proposed method to other state-of-the-art norm-based methods, including HaLRTC \cite{liu2013tensor}, SiLRTC-TT \cite{bengua2017efficient}, TNN \cite{zhang2014novel}, TRNNM \cite{Yu2019tensor} and FFWTensor \cite{guo2017efficient}. To evaluate these methods, extensive experiments are conducted on three visual-data sets: (i) A hyperspectral image (HSI)\footnote{Available at \url{http://www.ehu.eus/ccwintco/index.php/Hyperspectral_Remote_Sensing_Scenes}} of size $200\times 200\times 80$, which records the area of urban landscape; (ii) The Train-video\footnote{Available at \url{https://www.youtube.com/watch?v=6mcDsY0TwcA}} which consists of 80 color frames of size $72\times128\times3$, presented by a tensor of size $72\times128\times3\times80$; (iii) The AT\&T ORL\footnote{Available at \url{http://www.uk.research.att.com/facedatabase.html.}} face data set which consists of 10 different images of size $32\times32$ for each of 40 distinct subjects, presented by a tensor of size $32\times32\times10\times40$. Since reshaping the visual data into high-order tensor significantly improve the performance of the TT/TR-based methods (i.e. proposed method, TRNNM and SiLRTC-TT), which is illustrated in our experiments and previous works \cite{ bengua2017efficient, Yu2019tensor}, we reshape these three visual-data sets into high-order tensors for the TT/TR-based methods. Specifically, the HSI, Train-video and AT\&T ORL face data are reshaped into high-order tensors of size $ 10\times 20\times 10\times 20\times 8\times 10 $ (6D), $ 8\times 9\times 8\times 16\times 3\times 8\times 10 $ (7D) and $ 4\times 8\times 4\times 8\times 10\times 4\times 10 $ (7D), respectively. In our experiments, the parameters of the compared methods are set according to the corresponding paper such to achieve the best results.

As shown in Fig. \ref{plot:HSI}, \ref{plot:Train}, \ref{plot:ORL}, RSE, PSNR, and Runtime are used to evaluate the performance of each method on these three visual-data sets under uniformly random missing ratios \{80\%, 85\%, 90\%, 95\%\}. Observe that, in our considering cases, the proposed method outperforms the other methods at a small time-cost. Better results of RSE and PSNR are benefited from the powerful ability of a more balanced unfolding scheme in catching the global information. Smaller time-cost is caused by the efficiently-utilization of sparsity structure and rank-one SVD operation during iteration. Though FFWTensor method spends comparable time-cost with the proposed method, it fails to achieve good performance as the proposed method in most cases, especially in high missing-ratio cases \{90\%, 95\%\}. The other methods (i.e. HaLRTC, SiLRTC-TT, TNN, TRNNM) can achieve comparable results with the proposed method in some cases, however, require largely time-cost. Moreover, for HaLRTC, SiLRTC-TT, and TRNNM, the computational expensive determination of several weighting-parameters significantly increase their time-cost. These imply that, compared to the proposed method, other norm-based completion methods are not good choices for the large-scale data in practical applications. In addition, the visual results of each method on these three data sets are shown in Fig. \ref{fig:HSI}, \ref{fig:Train}, \ref{fig:ORL}. Observe that the proposed method obtains the recovery images with a better resolution and captures much more detailed information, e.g. wheel, beard, and eyes.
 
\begin{figure*}
\centering
\subfloat[]{\includegraphics[width=0.3\textwidth]
{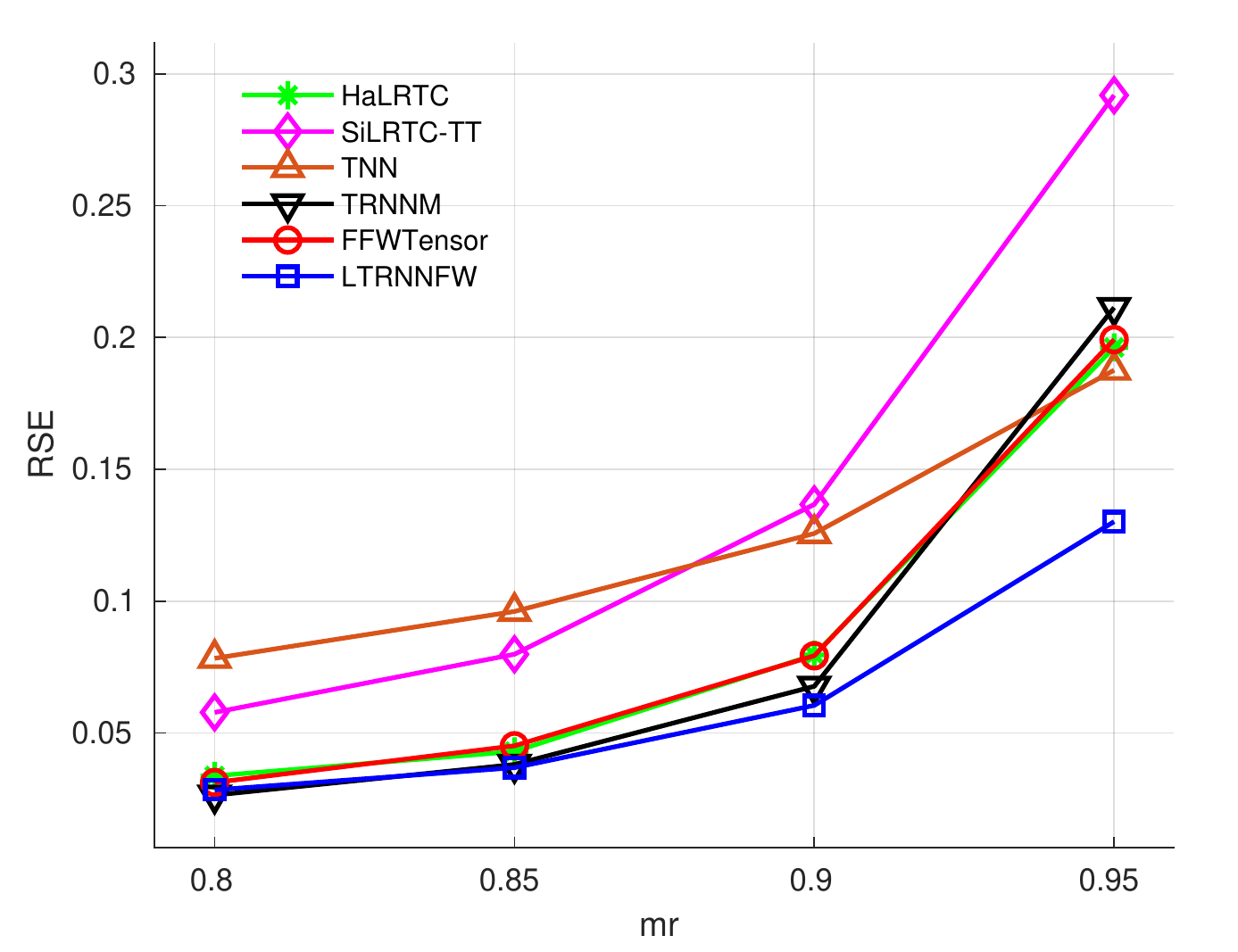}}
\subfloat[]{\includegraphics[width=0.3\textwidth]
{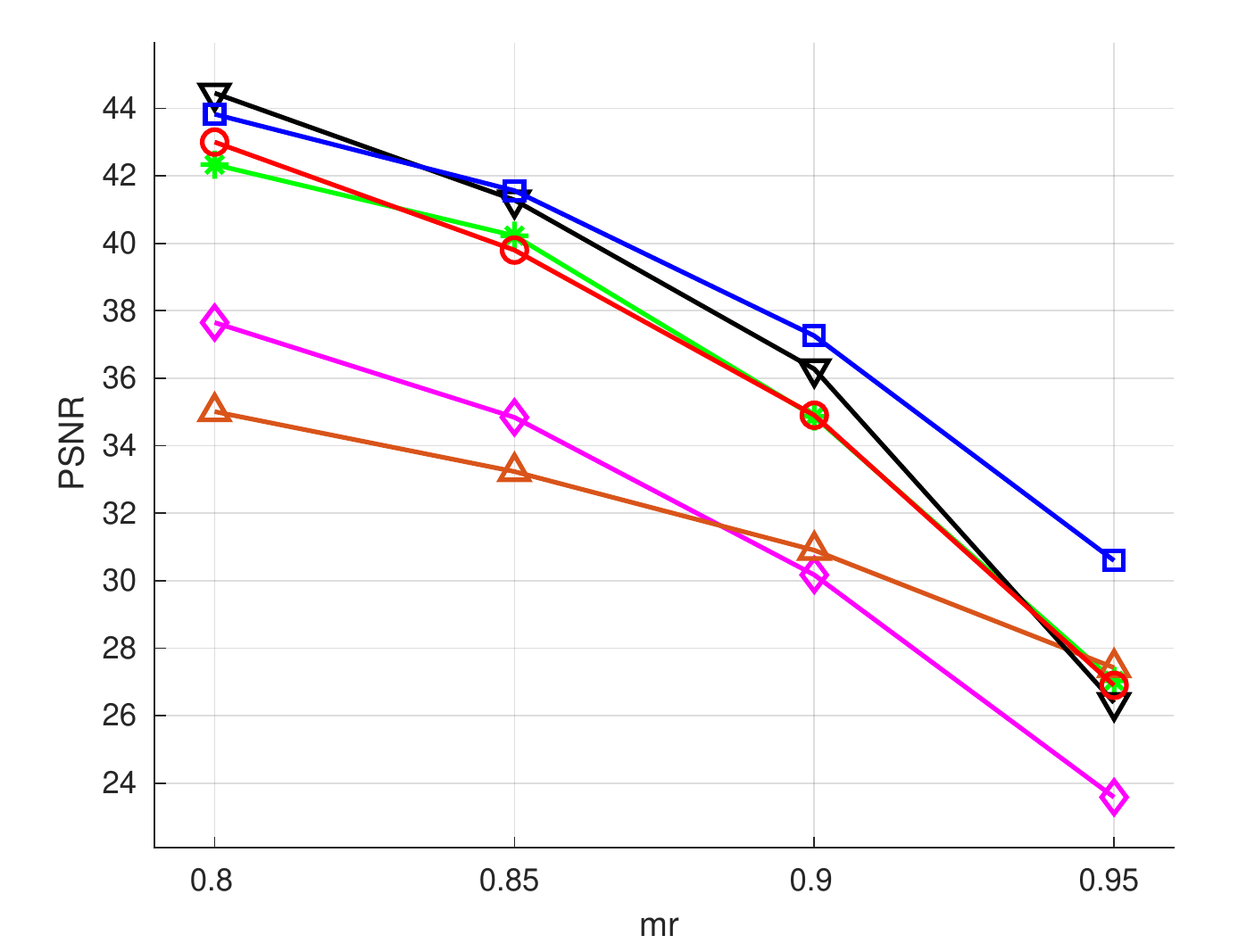}}
\subfloat[]{\includegraphics[width=0.3\textwidth]
{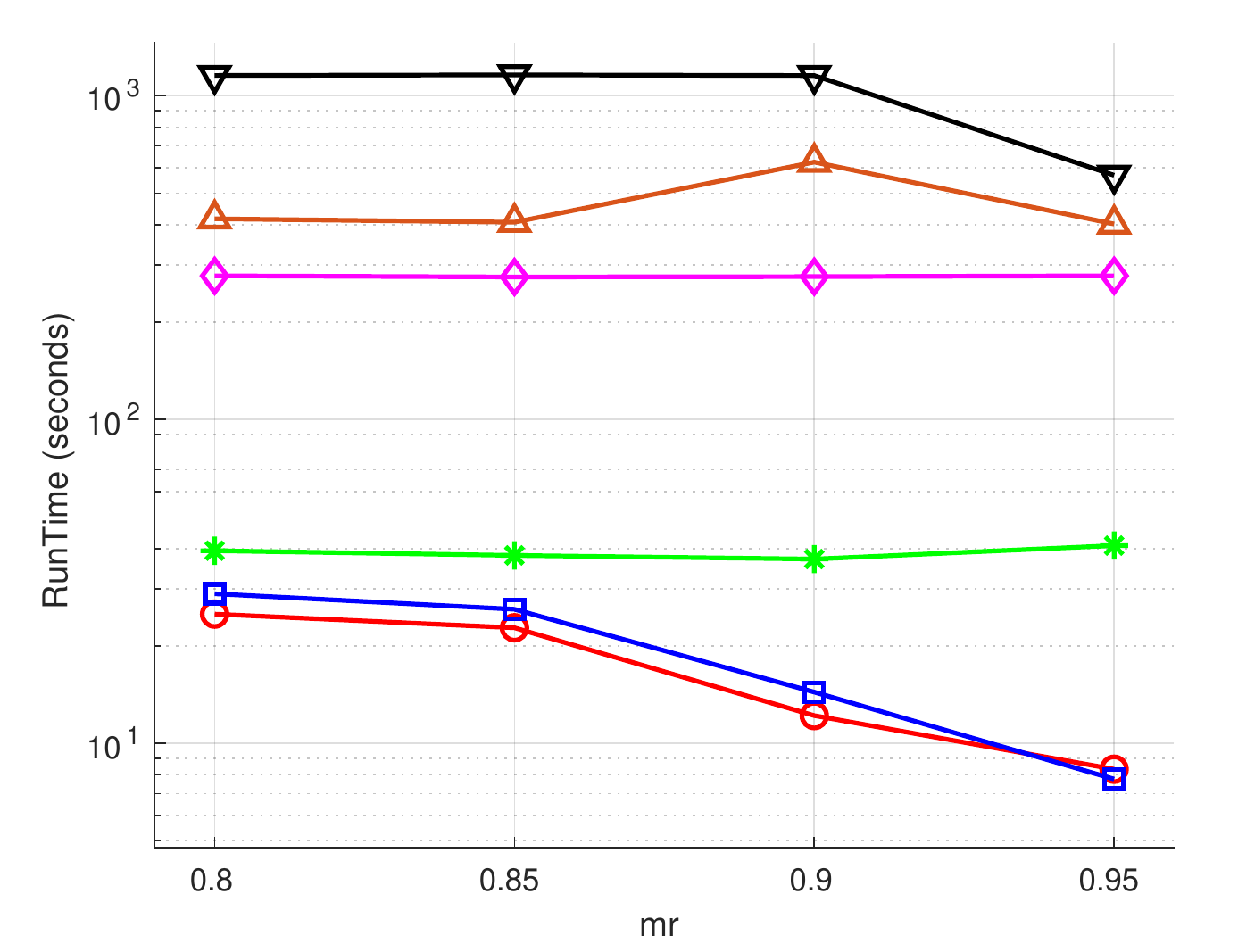}}
\caption{Comparison of  RSE, PSNR and runtime (seconds) on HSI images under varying missing ratios.}
\label{plot:HSI}
\end{figure*}

\begin{figure*}
\centering
\subfloat[]{\includegraphics[width=0.3\textwidth]
{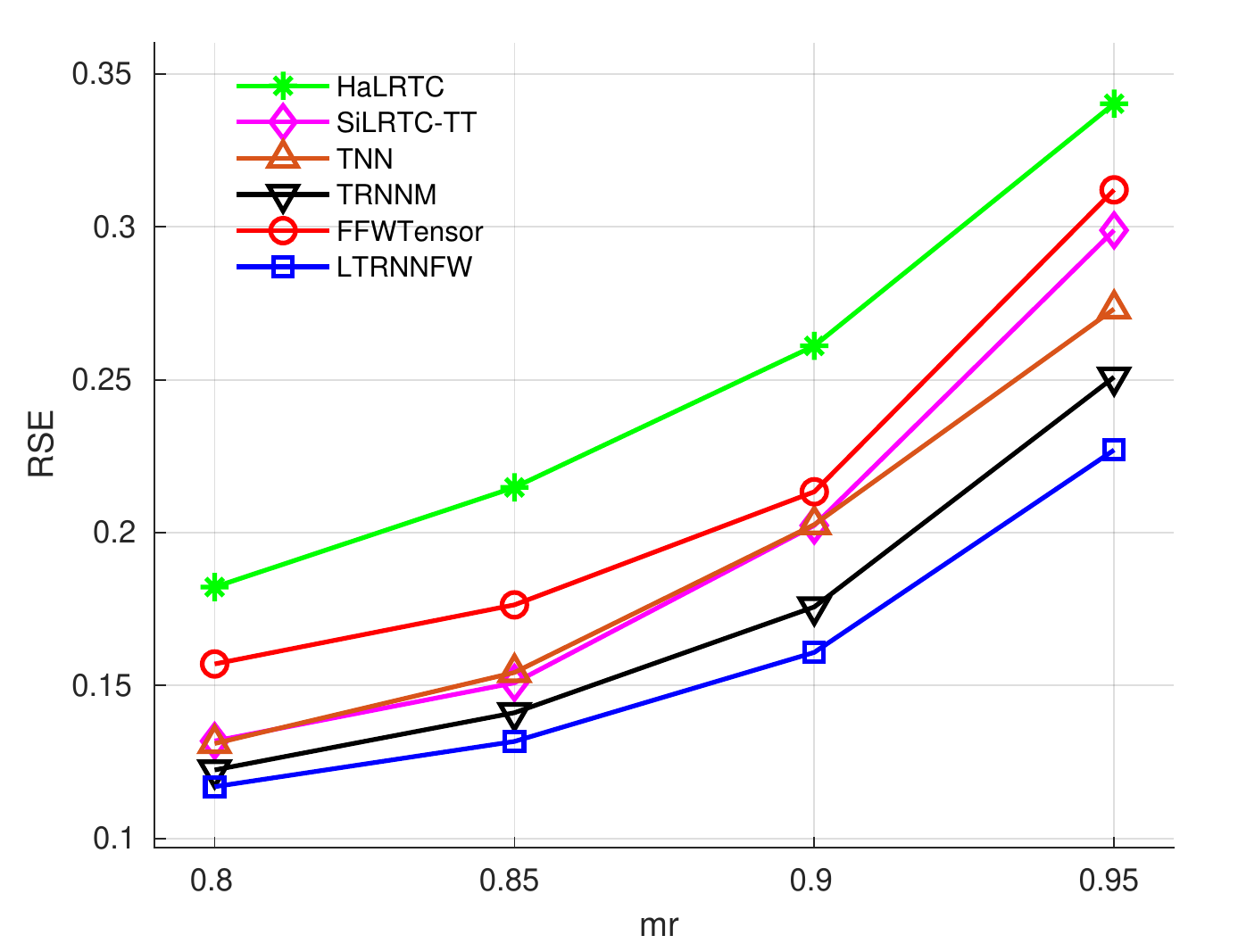}}
\subfloat[]{\includegraphics[width=0.3\textwidth]
{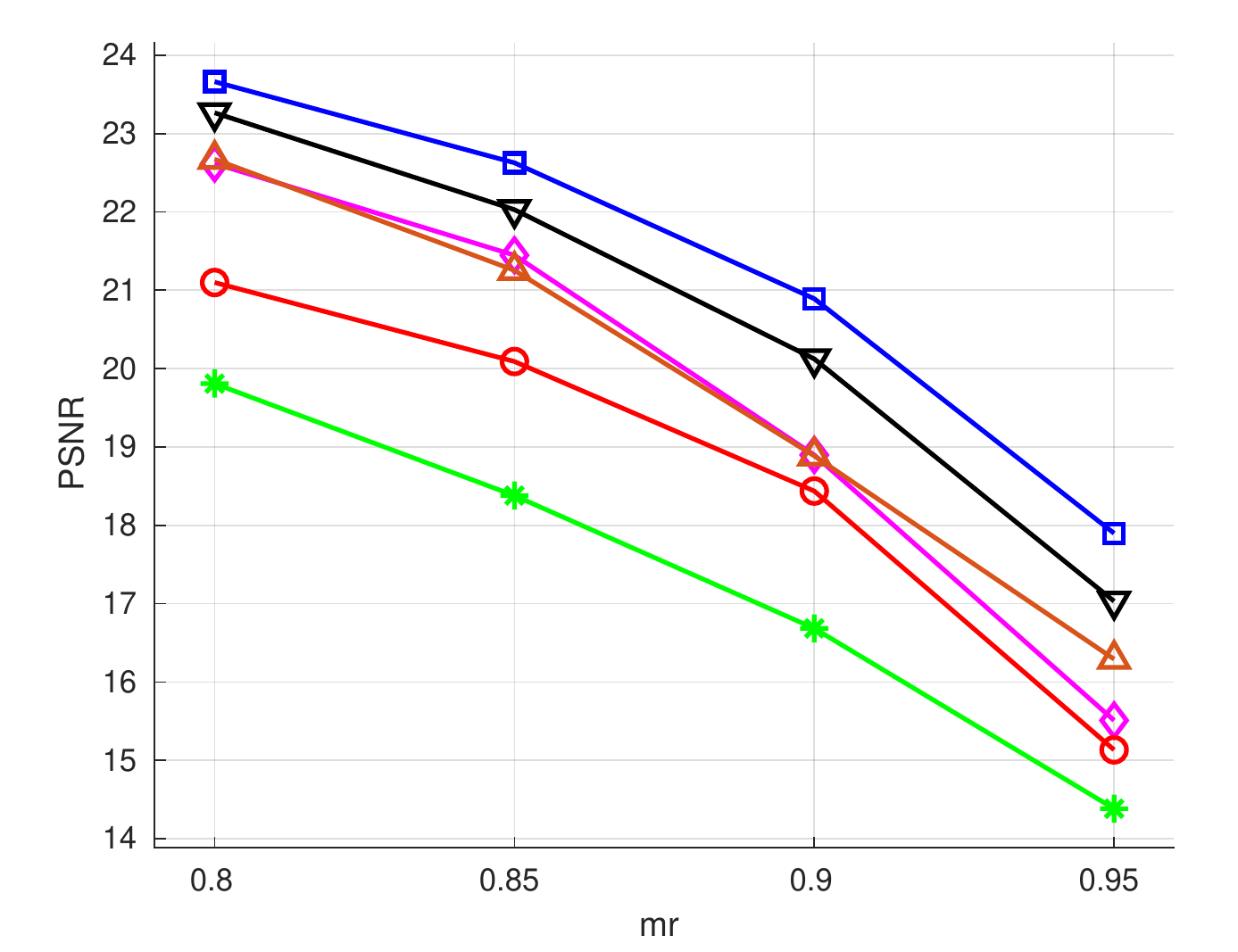}}
\subfloat[]{\includegraphics[width=0.3\textwidth]
{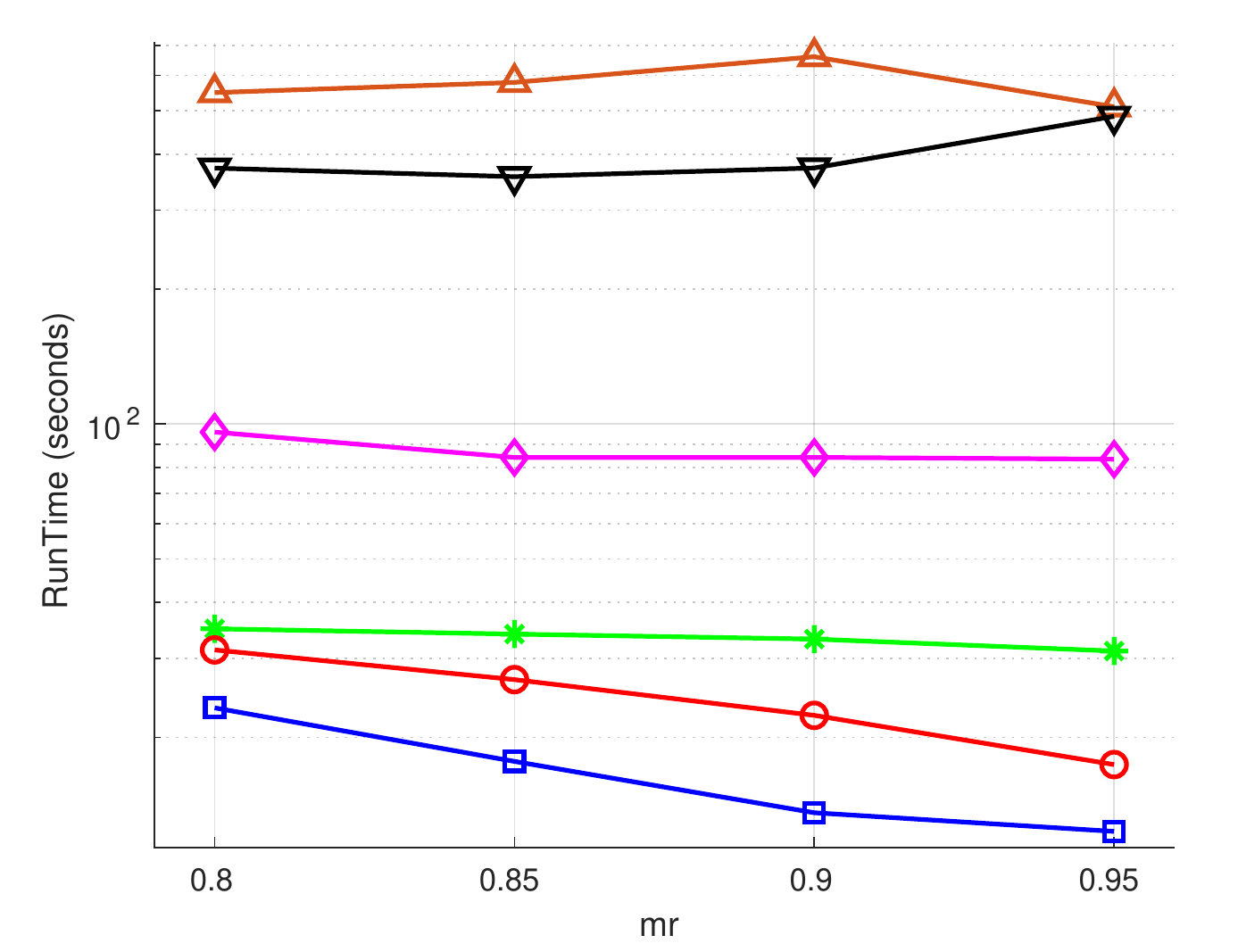}}
\caption{Comparison of RSE, PSNR and runtime (seconds) on the Train video under varying missing ratios.}
\label{plot:Train}
\end{figure*}

\begin{figure*}
\centering
\subfloat[]{\includegraphics[width=0.3\textwidth]
{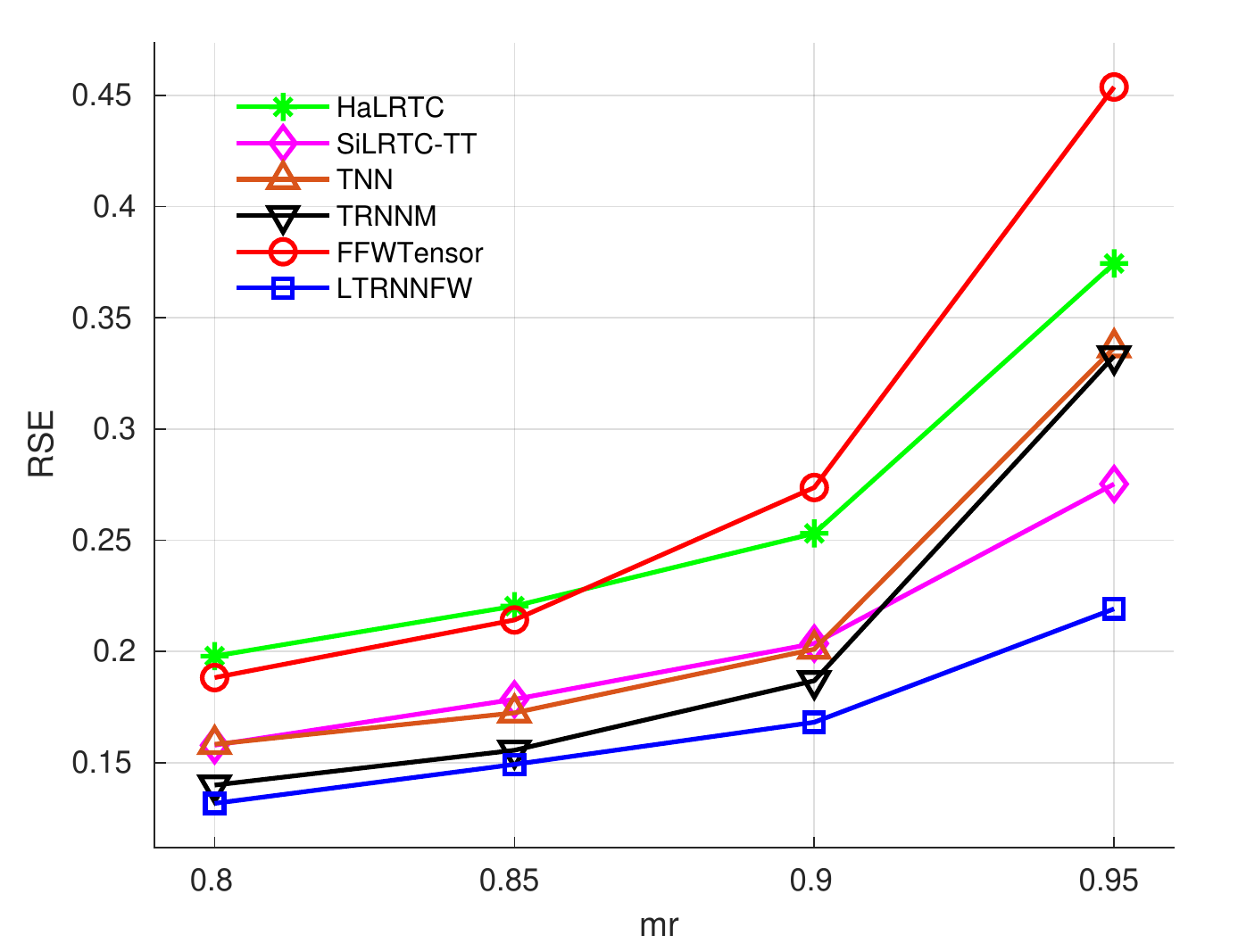}}
\subfloat[]{\includegraphics[width=0.3\textwidth]
{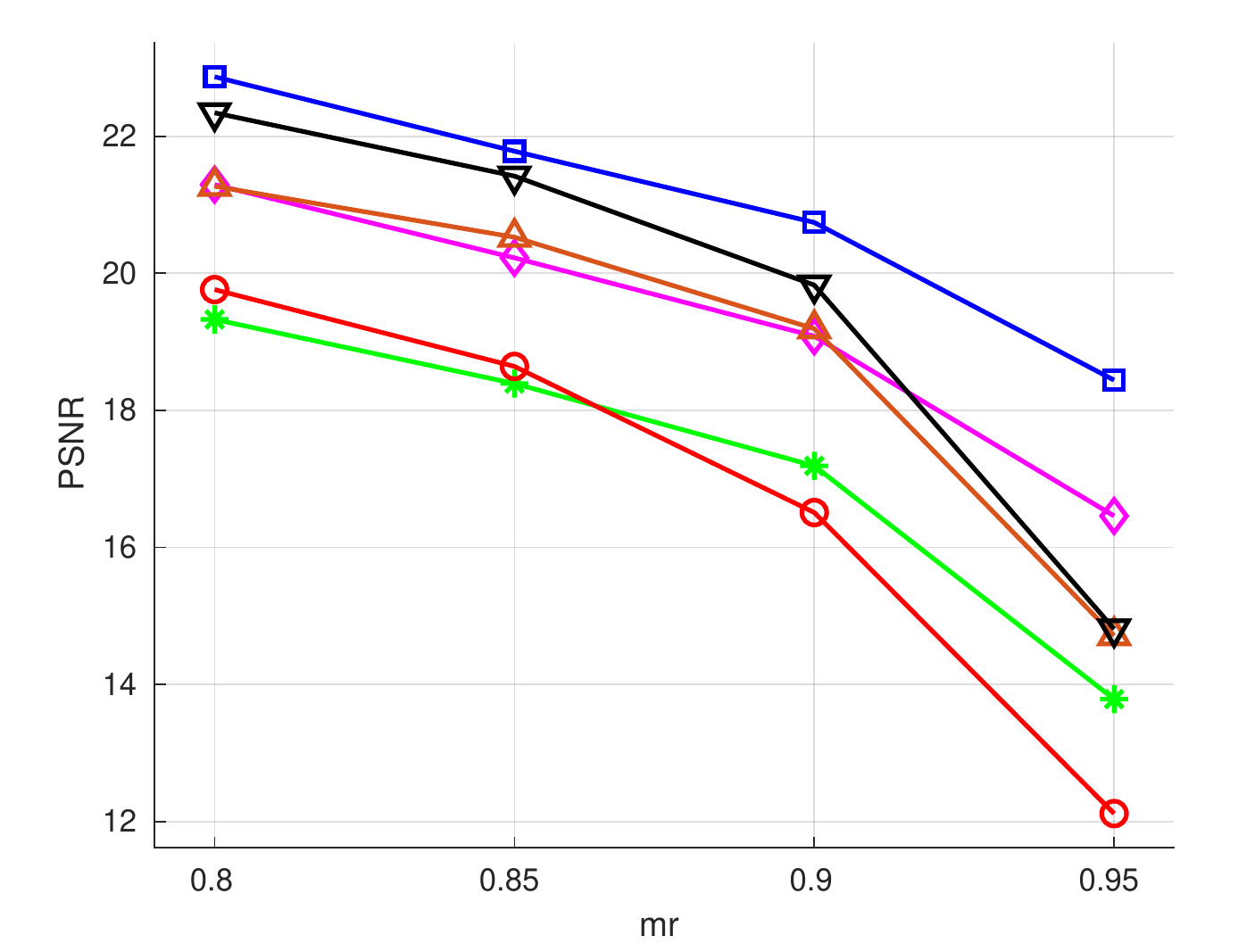}}
\subfloat[]{\includegraphics[width=0.3\textwidth]
{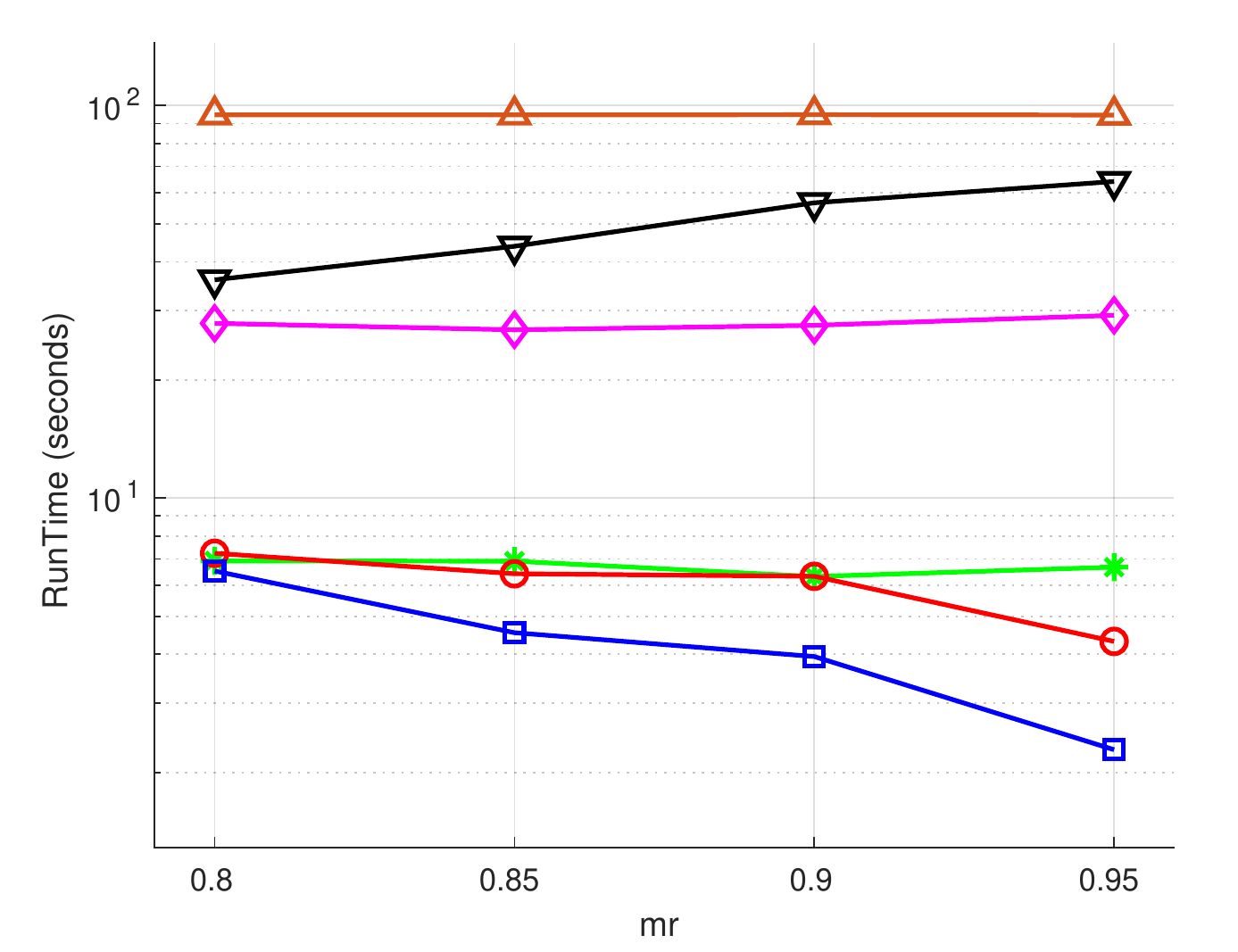}}
\caption{Comparison of  RSE, PSNR and runtime (seconds) on AT\&T ORL images under varying missing ratios.}
\label{plot:ORL}
\end{figure*}

\begin{figure*}
\begin{center}
\begin{minipage}[b]{0.17\textwidth}\centering
\includegraphics[width=1\linewidth]{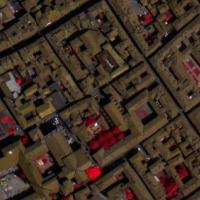}\caption*{Original}
\end{minipage}\vspace{0.1cm}
\begin{minipage}[b]{0.17\textwidth}\centering
\includegraphics[width=1\linewidth]{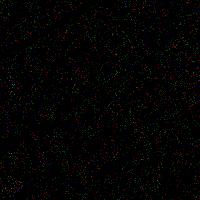}\caption*{Observation}
\end{minipage}
\begin{minipage}[b]{0.17\textwidth}\centering
\includegraphics[width=1\linewidth]{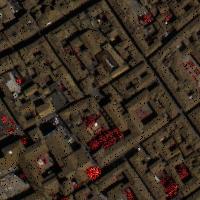}\caption*{HaLRTC}
\end{minipage}
\begin{minipage}[b]{0.17\textwidth}\centering
\includegraphics[width=1\linewidth]{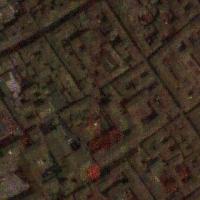}\caption*{SiLRTC-TT}
\end{minipage}\\
\begin{minipage}[b]{0.17\textwidth}\centering
\includegraphics[width=1\linewidth]{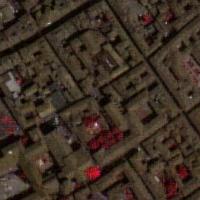}\caption*{TNN}
\end{minipage}
\begin{minipage}[b]{0.17\textwidth}\centering
\includegraphics[width=1\linewidth]{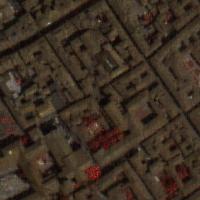}\caption*{TRNNM} 
\end{minipage}
\begin{minipage}[b]{0.17\textwidth}\centering 
\includegraphics[width=1\linewidth]{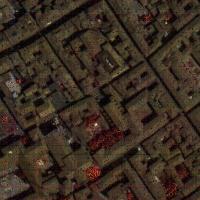}\caption*{FFWTensor}
\end{minipage}
\begin{minipage}[b]{0.17\textwidth}\centering
\includegraphics[width=1\linewidth]{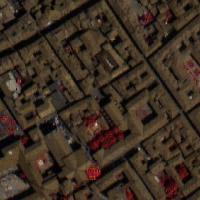}\caption*{{\bf LTRNNFW}} 
\end{minipage}
\caption{The visual results of each algorithm on the HSI data with the uniformly missing ratio of 95\%. The recovery results are shown in RGB format by picking the three bands of 70, 40, 10.}
\label{fig:HSI}
\end{center}
\end{figure*}

\begin{figure*}
\begin{center}
\begin{minipage}[b]{0.17\textwidth}\centering
\includegraphics[width=1\linewidth]{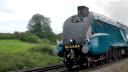}\caption*{Original}
\end{minipage}\vspace{0.1cm}
\begin{minipage}[b]{0.17\textwidth}\centering
\includegraphics[width=1\linewidth]{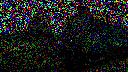}\caption*{Observation}
\end{minipage}
\begin{minipage}[b]{0.17\textwidth}\centering
\includegraphics[width=1\linewidth]{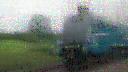}\caption*{HaLRTC}
\end{minipage}
\begin{minipage}[b]{0.17\textwidth}\centering
\includegraphics[width=1\linewidth]{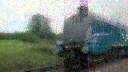}\caption*{SiLRTC-TT}
\end{minipage}\\
\begin{minipage}[b]{0.17\textwidth}\centering
\includegraphics[width=1\linewidth]{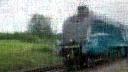}\caption*{TNN}
\end{minipage}
\begin{minipage}[b]{0.17\textwidth}\centering
\includegraphics[width=1\linewidth]{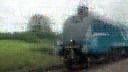}\caption*{TRNNM} 
\end{minipage}
\begin{minipage}[b]{0.17\textwidth}\centering 
\includegraphics[width=1\linewidth]{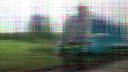}\caption*{FFWTensor}
\end{minipage}
\begin{minipage}[b]{0.17\textwidth}\centering
\includegraphics[width=1\linewidth]{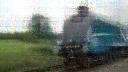}\caption*{{\bf LTRNNFW}} 
\end{minipage}
\caption{The visual results of each algorithm on the Train video with the uniformly missing ratio of 80\%. One frame of the video is picked to show the recovery results.}
\label{fig:Train}
\end{center}
\end{figure*}

\begin{figure*}
\begin{center}
\begin{minipage}[b]{0.2\textwidth}\centering
\includegraphics[width=1\linewidth]{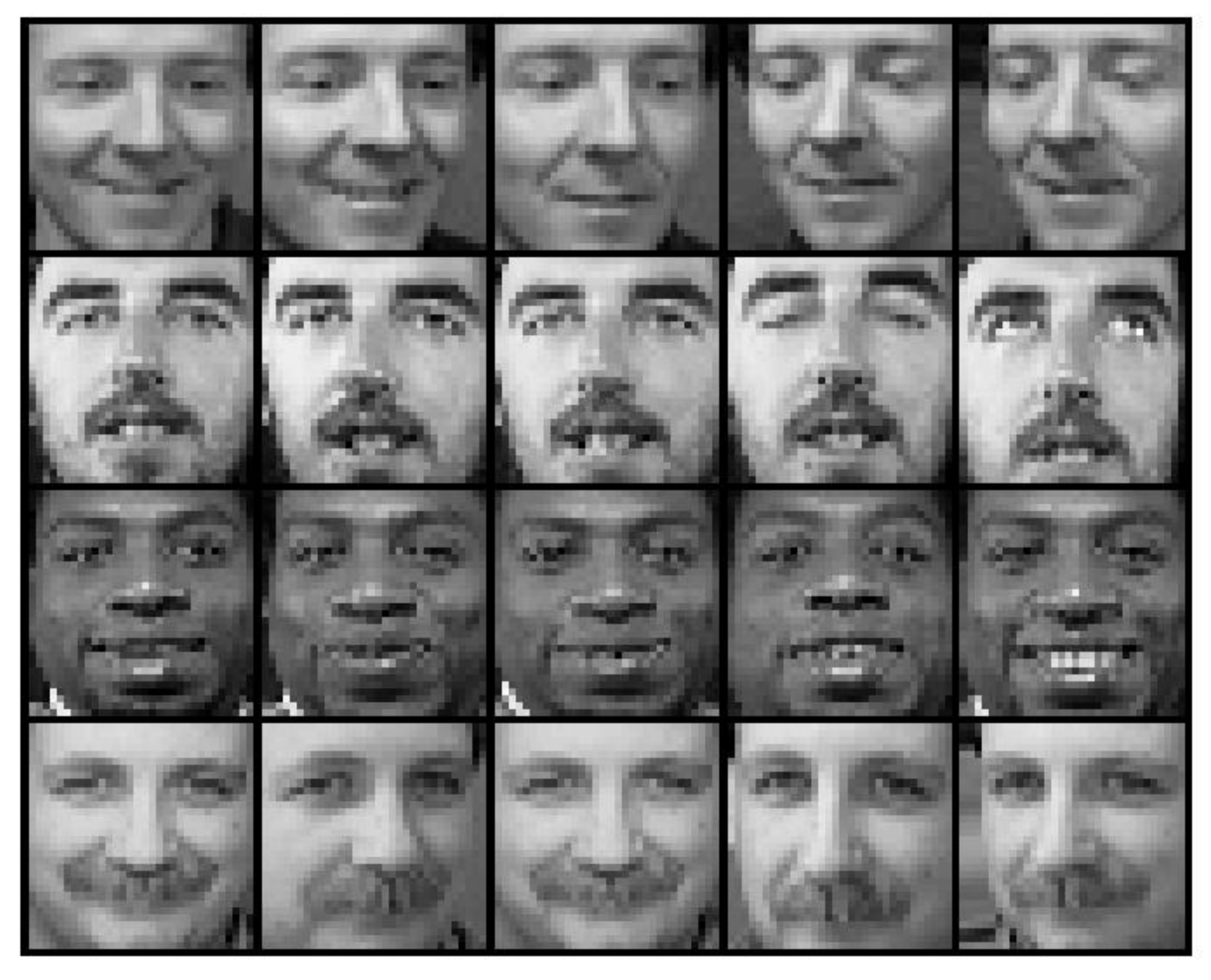}\caption*{\centering Original}
\end{minipage}
\begin{minipage}[b]{0.2\textwidth}\centering
\includegraphics[width=1\linewidth]{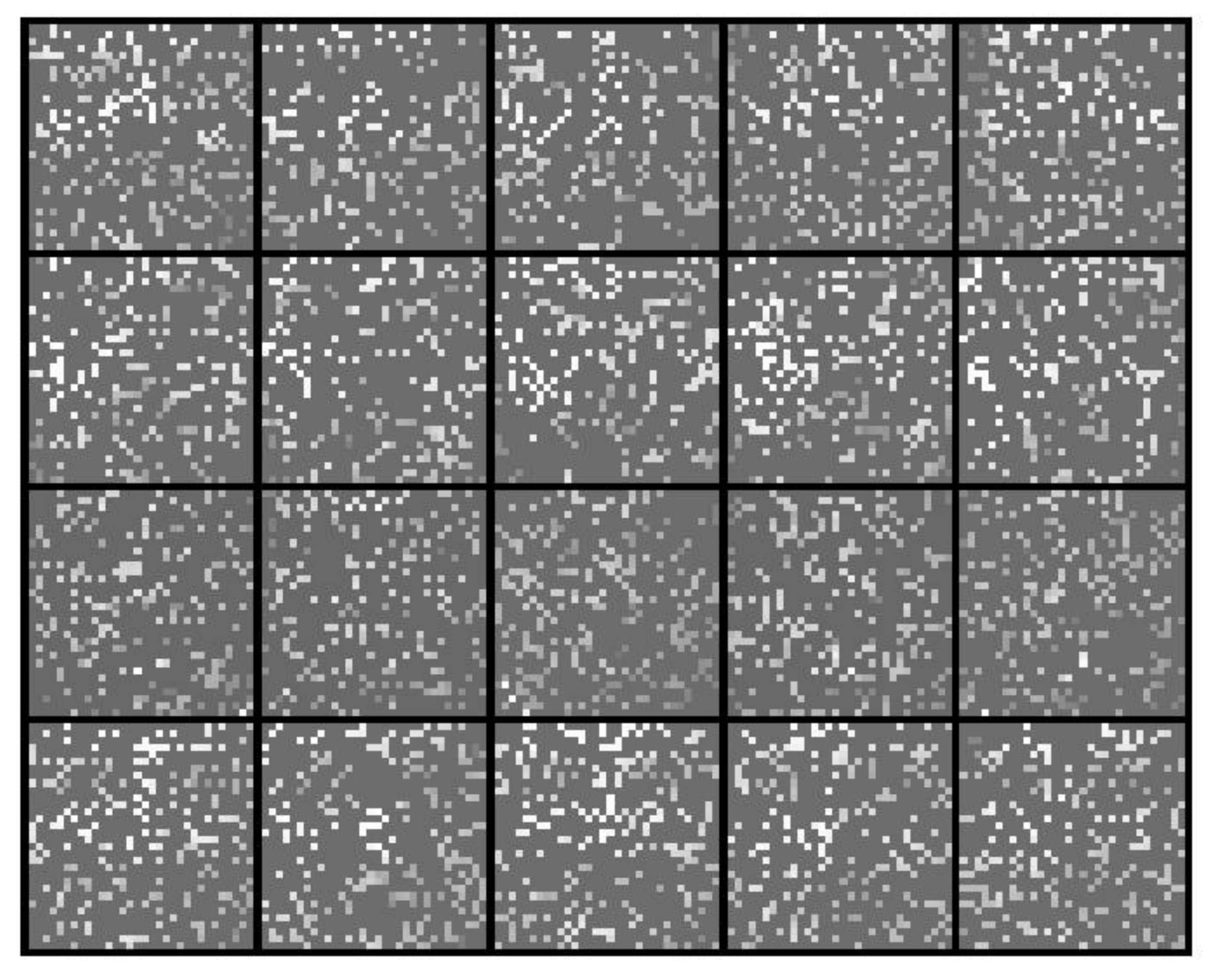}\caption*{Observation}
\end{minipage}
\begin{minipage}[b]{0.2\textwidth}\centering
\includegraphics[width=1\linewidth]{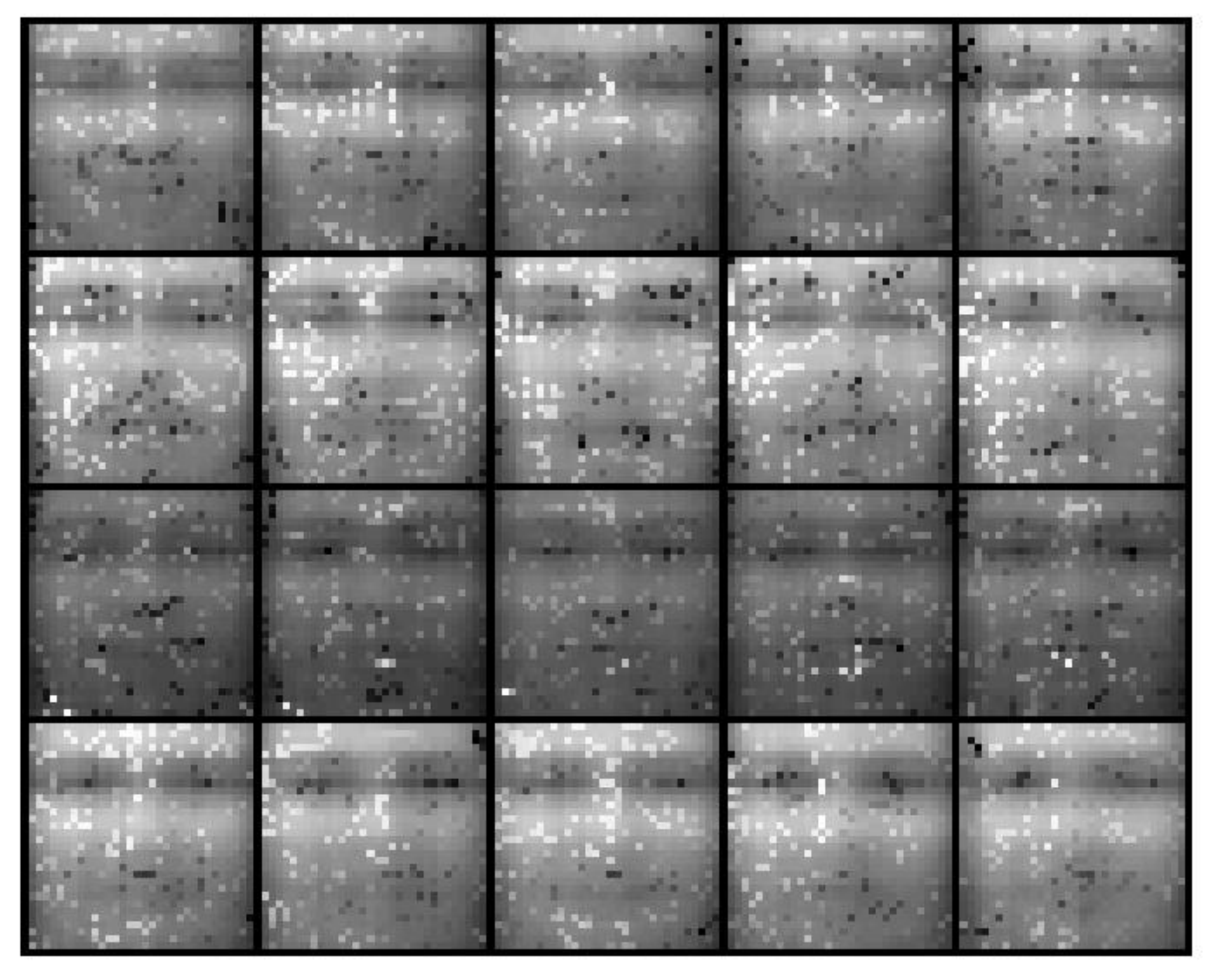}\caption*{HaLRTC}
\end{minipage}
\begin{minipage}[b]{0.2\textwidth}\centering
\includegraphics[width=1\linewidth]{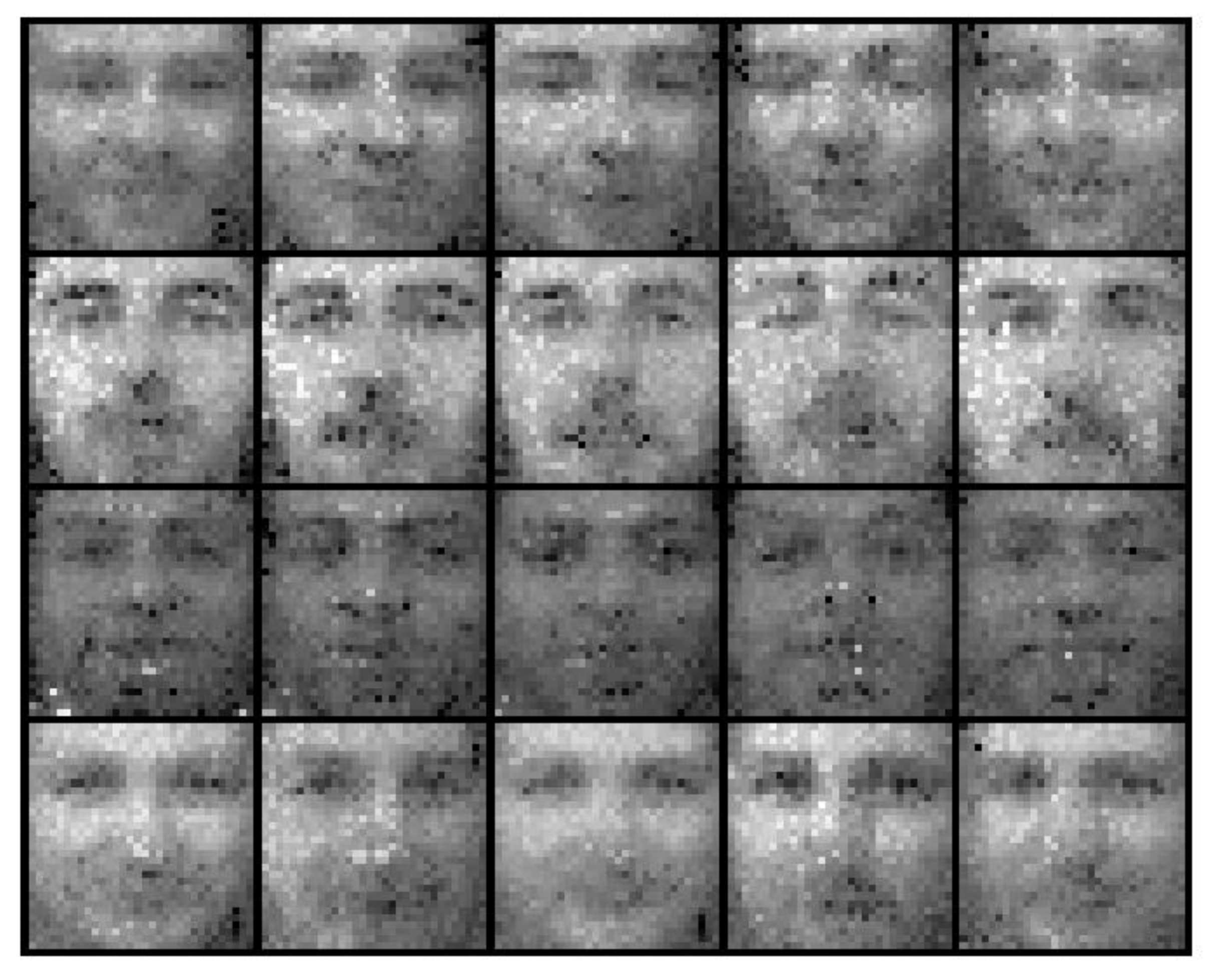}\caption*{SiLRTC-TT}
\end{minipage}\\
\begin{minipage}[b]{0.2\textwidth}\centering
\includegraphics[width=1\linewidth]{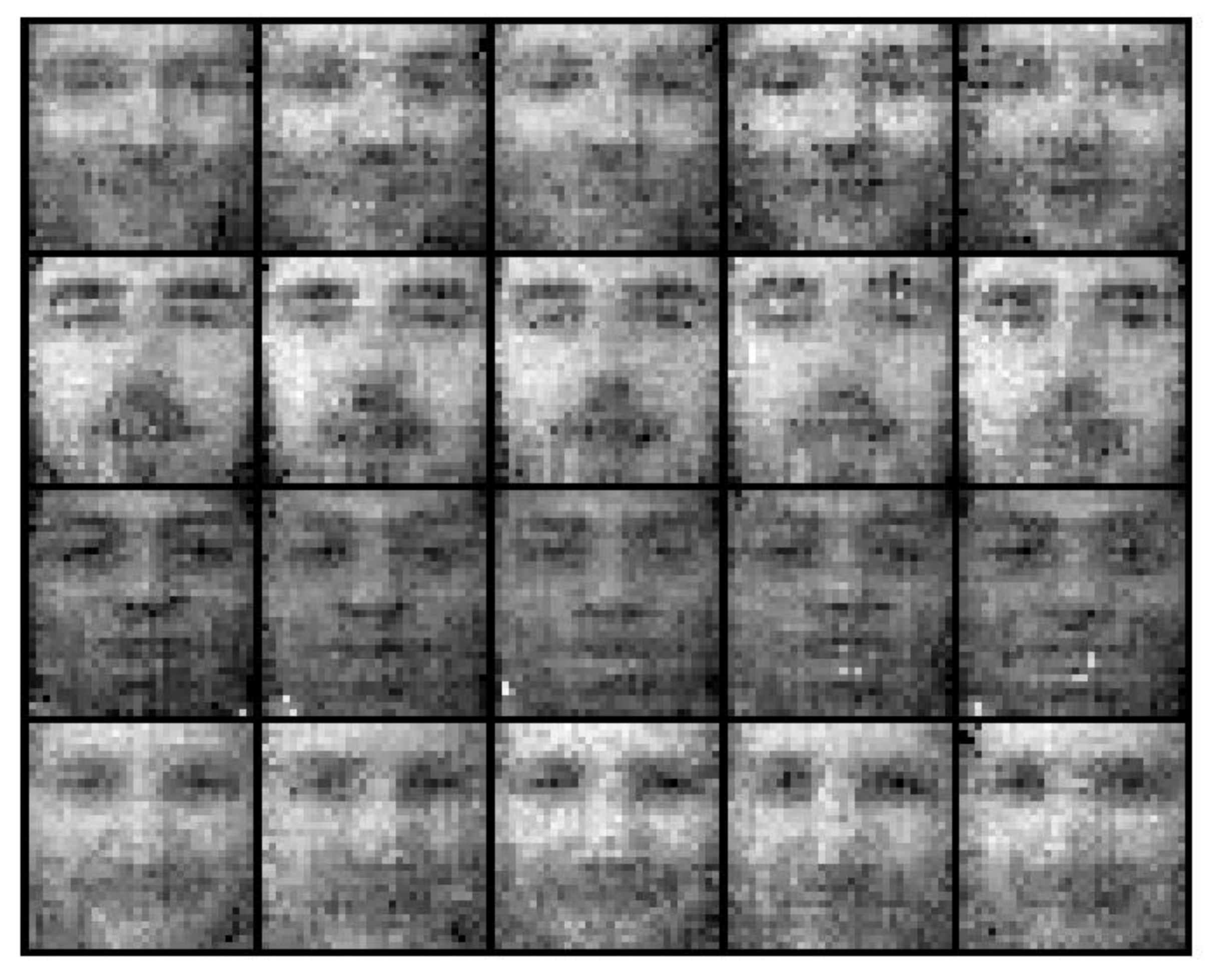}\caption*{TNN}
\end{minipage}
\begin{minipage}[b]{0.2\textwidth}\centering
\includegraphics[width=1\linewidth]{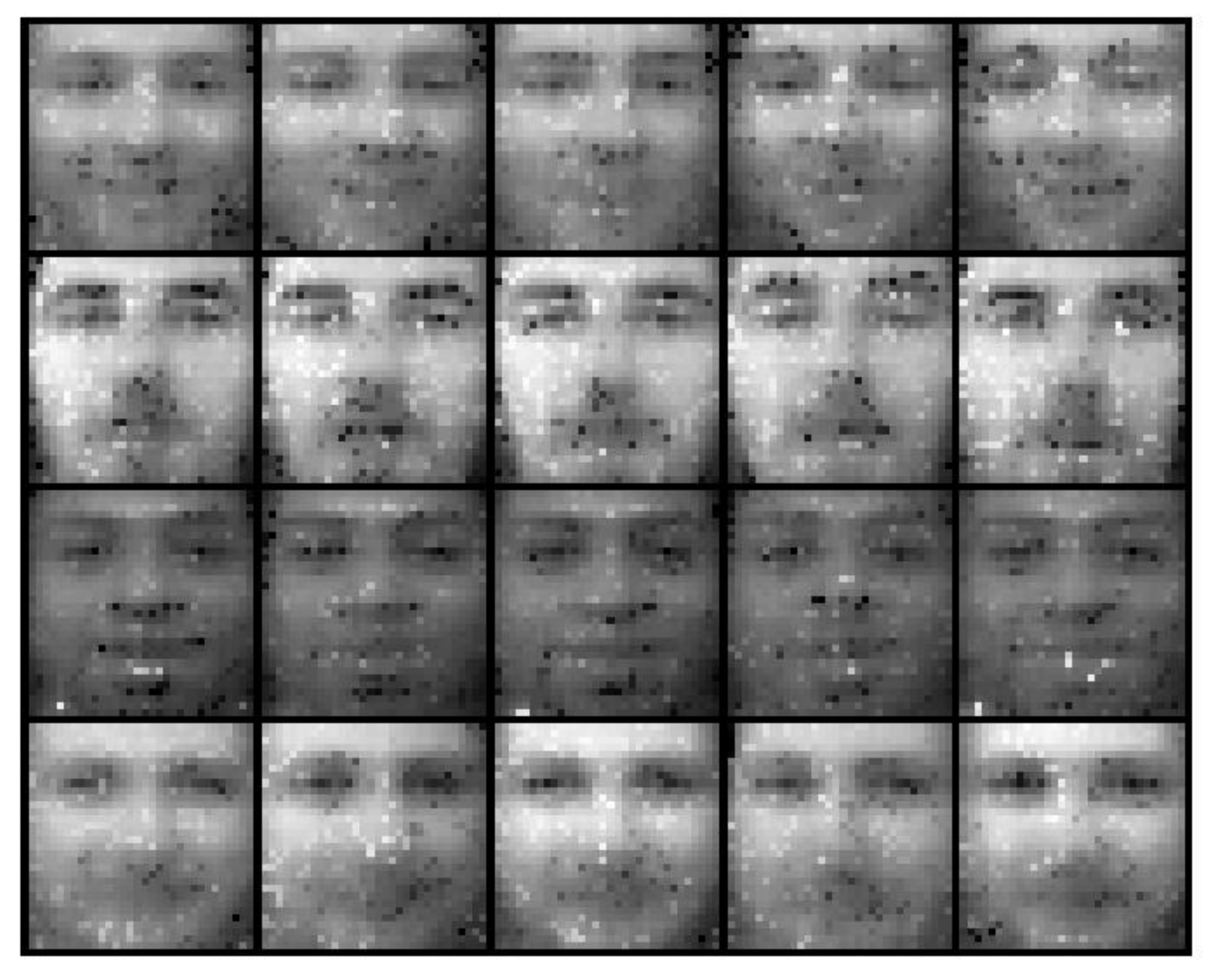}\caption*{TRNNM} 
\end{minipage}
\begin{minipage}[b]{0.2\textwidth}\centering 
\includegraphics[width=1\linewidth]{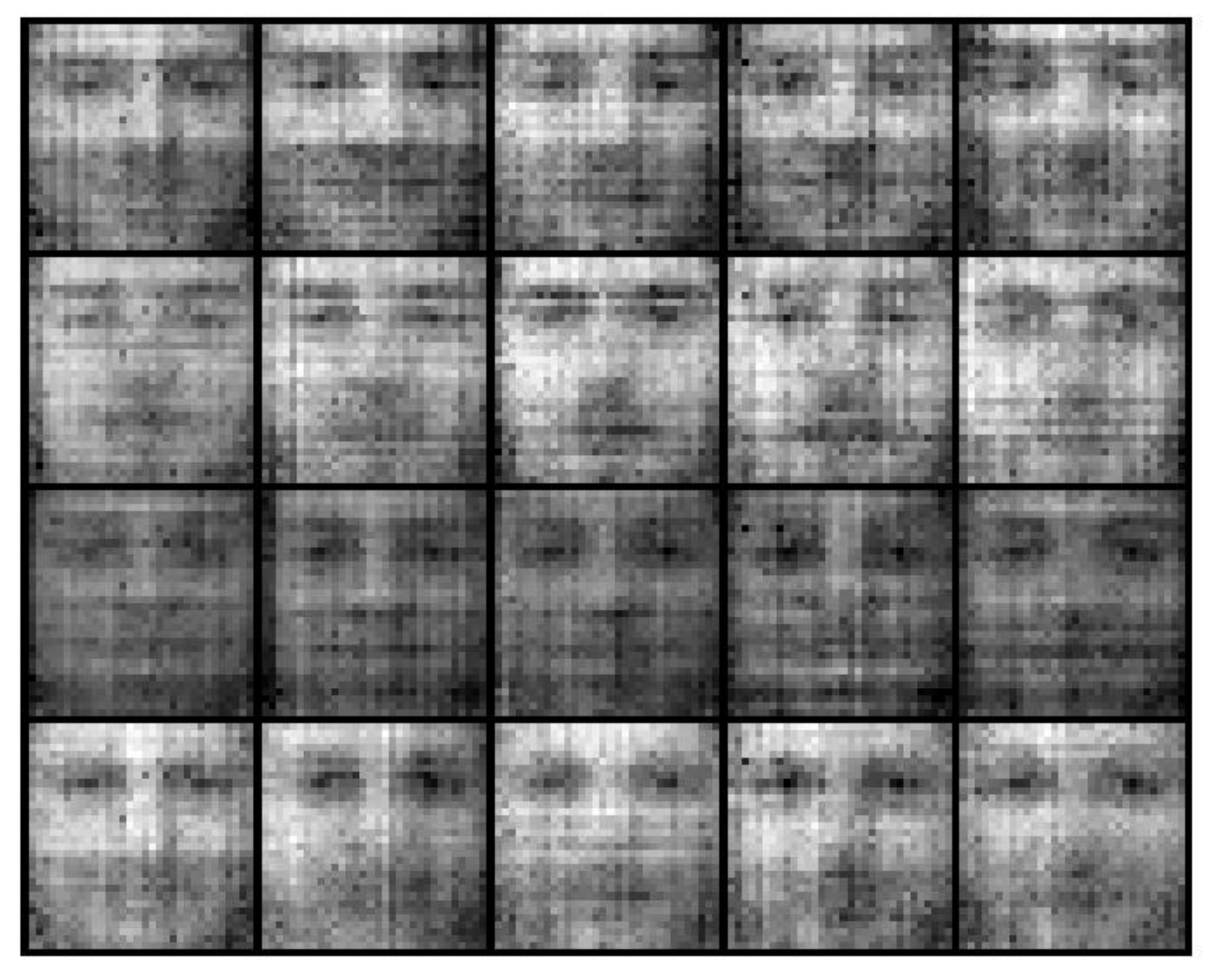}\caption*{FFWTensor}
\end{minipage}
\begin{minipage}[b]{0.2\textwidth}\centering
\includegraphics[width=1\linewidth]{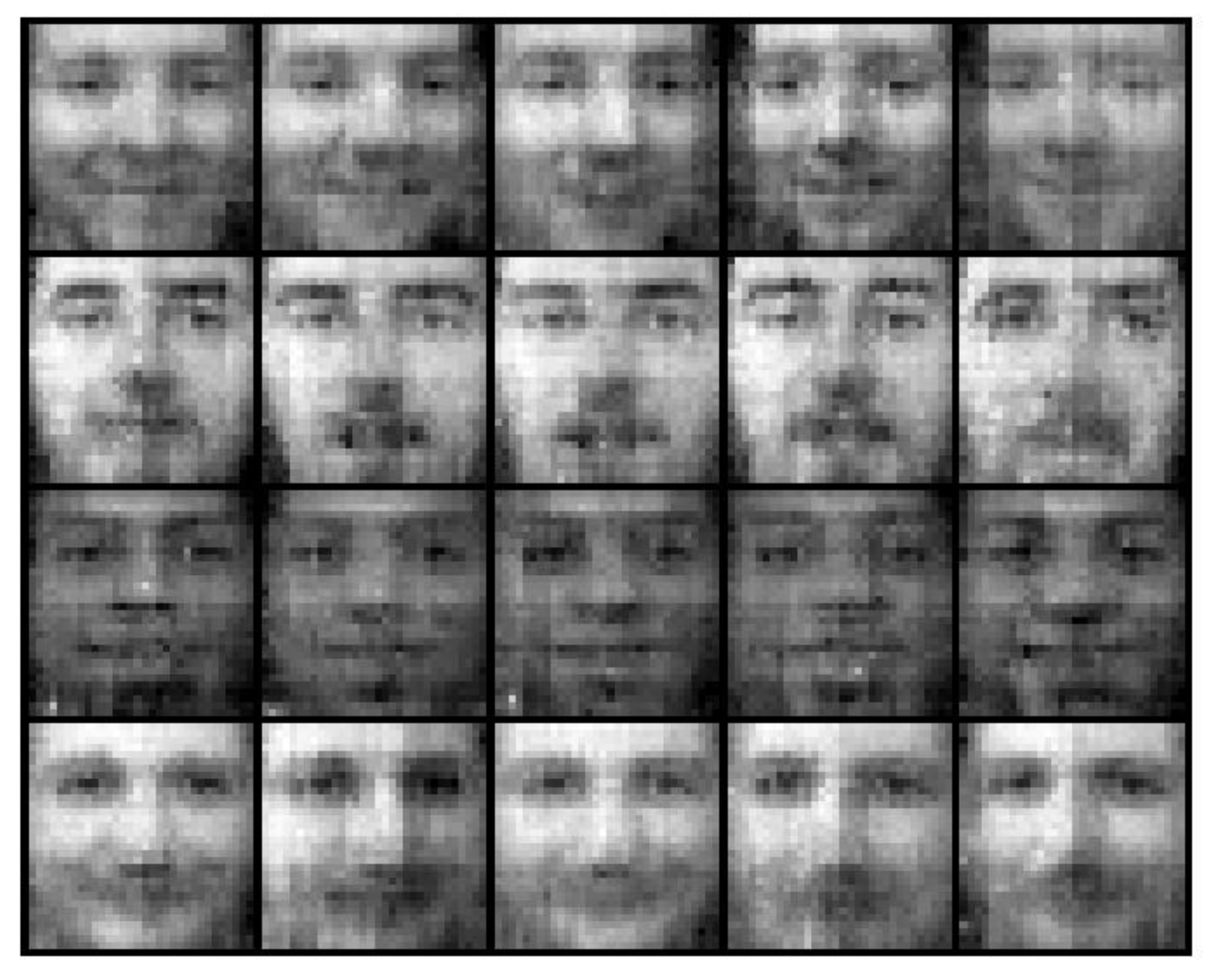}\caption*{{\bf LTRNNFW}} 
\end{minipage}
\caption{The visual results of each algorithm on the AT\&T ORL images with the uniformly missing ratio of 80\%. 20 images are picked to show the recovery results.}
\label{fig:ORL}
\end{center}
\end{figure*}

%% file: Conclusion.tex
\section{Conclusion}\label{conclusion}
In this paper, a new latent nuclear norm equipped with a more balanced unfolding scheme is defined for low-rank regularization, and an efficient Frank-Wolfe algorithm is developed for optimization by utilization of sparsity structure and rank-one SVD operation. We theoretically analyze that the proposed method is much more efficient over other norm-based methods in terms of both time and space, which is important for the memory-limited equipment in practical applications. Furthermore, extensive experimental results confirm that the proposed method can achieve state-of-the-art performance in visual-data inpainting at smaller costs of time and space.